\documentclass{article}

\usepackage{microtype}
\usepackage{graphicx}
\usepackage{subfigure}
\usepackage{booktabs} 

\usepackage{hyperref}

\usepackage[accepted]{icml2021}

\icmltitlerunning{Domain Generalization using Causal Matching}

\usepackage{amsfonts}       
\usepackage{nicefrac}       
\usepackage{microtype}      
\usepackage{amsmath}
\usepackage{amssymb}
\usepackage{amsthm}
\usepackage{bbm} 
\usepackage{bm} 
\usepackage{verbatim}
\usepackage{xcolor}
\usepackage{adjustbox}
\usepackage{multirow}
\usepackage{wrapfig}
\usepackage{enumitem}
\usepackage{thmtools,thm-restate}
\usepackage{nccmath}
\newif{\ifhidecomments}
\hidecommentsfalse

\ifhidecomments
	\newcommand{\amit}[1]{}
	\newcommand{\shruti}[1]{}
	\newcommand{\divyat}[1]{}
\else
    \newcommand{\amit}[1]{\textcolor{red}{Amit: #1}}
    \newcommand{\shruti}[1]{\textcolor{blue}{Shruti: #1}}
    \newcommand{\divyat}[1]{\textcolor{brown}{Divyat: #1}}
\fi

\newcommand{\figref}[1]{Figure~\ref{#1}}
\newcommand{\secref}[1]{Section~\ref{#1}}
\newcommand{\tabref}[1]{Table~\ref{#1}}

\newtheorem{definition}{Definition}

\newtheorem{assumption}{Assumption}
\newcommand{\indep}{\perp \! \! \! \perp}

\newcommand{\vecx}{\mathbf{x}}
\newcommand{\vecxc}{\mathbf{x}_c}
\newcommand{\vecphi}{\Phi}

\newcommand{\eps}{\epsilon}
\newcommand{\norm}[1]{\left\lVert#1\right\rVert}
\newcommand{\rotm}{\texttt{rotMNIST}}
\newcommand{\rotfm}{\texttt{rotFashionMNIST}}
\newcommand{\mdg}{\texttt{MatchDG}}
\newcommand{\epmatch}{\texttt{ERM-PerfMatch}}
\newcommand{\ermatch}{\texttt{ERM-RandMatch}}
\newcommand{\hybrid}{\texttt{MDGHybrid}}
\newcommand{\xhdr}[1]{\textbf{#1}}

\begin{document}

\twocolumn[
\icmltitle{Domain Generalization using Causal Matching}



\icmlsetsymbol{equal}{*}

\begin{icmlauthorlist}
\icmlauthor{Divyat Mahajan}{msr1}
\icmlauthor{Shruti Tople}{msr2}
\icmlauthor{Amit Sharma}{msr1}
\end{icmlauthorlist}

\icmlaffiliation{msr1}{Microsoft Research, India}
\icmlaffiliation{msr2}{Microsoft Research, UK.}
\icmlcorrespondingauthor{Divyat Mahajan}{divyatmahajan@gmail.com}

\icmlkeywords{Machine Learning, ICML}

\vskip 0.3in
]



\printAffiliationsAndNotice{}  

\begin{abstract}
 In the domain generalization literature,
 a common objective is to learn representations independent of the domain after conditioning on the class label. We show that this objective is not sufficient: there exist counter-examples where a model fails to generalize to unseen domains even after satisfying class-conditional domain invariance. We formalize this observation through a structural causal model and show the importance of modeling \textit{within-class} variations for generalization. Specifically, classes contain \textit{objects} that characterize specific causal features, and domains can be interpreted as interventions on these objects that change non-causal features. We highlight an alternative condition: inputs across domains should have the same representation if they are derived from the same object. Based on this objective, we propose  matching-based algorithms when base objects are observed (e.g., through data augmentation) and approximate the objective when objects are not observed (\mdg). Our simple matching-based algorithms are competitive to prior work on out-of-domain accuracy for rotated MNIST, Fashion-MNIST, PACS, and Chest-Xray datasets. Our method \mdg\ also recovers ground-truth \textit{object matches}: on MNIST and Fashion-MNIST, top-10 matches from \mdg\ have over 50\% overlap with ground-truth matches. 
\end{abstract}

\section{Introduction}
Domain generalization is the task of learning a machine learning model that can generalize to unseen data distributions, after training on more than one data distributions. For example, a model trained on hospitals in one region may be deployed to another, or an image classifier may be deployed on slightly rotated images. Typically, it is assumed that the different domains share some ``stable'' features whose relationship with the output is invariant across domains~\cite{piratla2020efficient} and the goal is to learn those features. 
A popular class of methods aim to learn representations that are independent of domain \textit{conditional on class}~\cite{li2018conddomaingen, li2018adversarialcada, ghifary2016scatter, hu2019discriminant}, based on evidence of their superiority~\cite{zhao2019learninginvariant} to methods that learn representations that are marginally independent of domain~\cite{muandet2013domain,ganin2016dann}. 

In this work, we show that the class-conditional domain-invariant objective for representations is insufficient. We provide counter-examples where a feature representation satisfies the objective but still fails to generalize to new domains, both theoretically and empirically. Specifically, when the distribution of the stable features to be learnt varies across domains, class-conditional objective is insufficient to learn the stable features (they are optimal only when the distribution of stable features is the same across domains). Differing distributions of stable features within the same class label is common in real-world datasets, e.g., in digit recognition, the stable feature \textit{shape} may differ based on people's handwriting, or medical images  may differ based on variation in body characteristics across people. Our investigation reveals the importance of considering within-class variation in the stable features.


To derive a better objective for domain generalization, we represent the within-class variation in stable features using a structural causal model, building on prior work~\cite{heinze2017conditional} from single-domain generalization.  
 Specifically, we construct a model for the data generation process that assumes each input is constructed from a mix of stable (\emph{causal}) and domain-dependent (\emph{non-causal}) features, and only the stable features cause the output. We consider domain as a special intervention that changes the non-causal features 
of an input, and posit that 
an ideal classifier should be based only on the causal features.  Using d-separation, we show that the correct objective is to build  a representation that is invariant conditional on each \textit{object}, where an object is defined as  a set of inputs that share the same causal features (e.g., photos of the same person from different viewpoints or augmentations of an image in different rotations, color or background).  
When the object variable is observed (e.g., in self-collected data or by dataset augmentation), we propose a \textit{perfect-match}  regularizer for domain generalization that minimizes the distance between representations of the same object across domains. 

In practice, however, the underlying  objects  are not always known. We therefore propose an approximation that aims to learn which inputs share the same object, under the assumption that inputs from the same class have more similar causal features than those from different classes. 
Our algorithm, \mdg\ is an iterative algorithm that starts with randomly matched inputs from the same class and builds a representation  using contrastive learning such that inputs sharing the same causal features are closer to one another.  
While past work has used contrastive loss to regularize the empirical risk minimization (ERM) objective~\cite{dou2019masf}, we demonstrate the importance of a two-phase method that first learns a representation independent of the ERM loss, so that classification loss does not interfere with the learning of stable features. 
In datasets with data augmentations, we extend \mdg\ to also use the perfect object matches obtained from pairs of original and augmented images (\hybrid). 

We evaluate our matching-based methods on rotated MNIST and Fashion-MNIST,  PACS and  Chest X-ray datasets. On all datasets, the simple methods \mdg\ and \hybrid\ are competitive to state-of-the-art methods for out-of-domain accuracy. On the rotated MNIST and Fashion-MNIST datasets where the ground-truth objects are known, \mdg\  learns to makes the representation more similar to their ground-truth  matches (about 50\% overlap for top-10 matches), even though the method does not have access to them.  Our results with simple matching methods show the importance of enforcing the correct invariance condition. 

\xhdr{Contributions.} To summarize, our contributions include: 
\textbf{1). }An object-invariant  condition for domain generalization that highlights a key limitation of previous approaches, 
\\
\textbf{2). } When object information is not available, a two-phase, iterative algorithm to approximate object-based matches.
%
Also, the code repository can be accessed at: \url{https://github.com/microsoft/robustdg}

\section{Related Work}
\textbf{Learning common representation.} 
To learn a generalizable classifier, several methods enforce the learnt representation $\Phi(\vecx)$ to be independent of domain marginally or conditional on class label, using divergence measures such as maximum mean discrepancy~\cite{muandet2013domain, li2018adversarialfeature, li2018conddomaingen}, adversarial training with a domain discriminator~\cite{ganin2016dann, li2018adversarialcada, albuquerque2020generalizing}, discriminant analysis~\cite{ghifary2016scatter,hu2019discriminant}, and other techniques~\cite{ghifary2015multitaskae}.

Among them, several works~\cite{zhao2019learninginvariant, johansson2019supportdomaininvariant, akuzawa2019adversarial} show that the class-conditional methods~\cite{li2018conddomaingen, li2018adversarialcada, ghifary2016scatter, hu2019discriminant} are better than those that enforce marginal domain-invariance of features~\cite{muandet2013domain, ganin2016dann, li2018adversarialfeature, albuquerque2020generalizing}, whenever there is a varying distribution of class labels across domains. We show that the class-conditional invariant is also not sufficient for generalizing to unseen domains. 

\textbf{Causality and domain generalization}. Past work has shown the connection between causality and generalizable predictors~\cite{peters2016causal,christiansen2020causal}. There is work on use of causal reasoning for domain adaptation~\cite{gong2016domainadapt,heinze2017conditional,magliacane2018domainadaptcausal,rojas2018invarianttheory} that assumes $Y \to X$ direction and other work~\cite{arjovsky2019irm, peters2016causal} on connecting causality that assumes $X \to Y$. Our SCM model unites these streams by introducing $Y_{true}$ and labelled $Y$ and develop an invariance condition for domain generalization that is valid under both interpretations. 
Perhaps the closest to our work is by \cite{heinze2017conditional} who use the object concept in single-domain datasets for better generalization. We extend their SCM to the multi-domain setting and use it to show the inconsistency of prior methods. In addition, while \cite{heinze2017conditional} assume objects are always observed, we also provide an algorithm for the case when objects are unobserved. 
 
 \textbf{Matching and Contrastive Loss.} Regularizers based on matching have been proposed for domain generalization. \cite{motiian2017ccsa} proposed matching representations of inputs from the same class. \cite{dou2019masf} used a contrastive (triplet) loss to regularize the ERM objective. In contrast to regularizing based on contrastive loss, our algorithm \mdg\ proceeds in two phases and learns a representation independent of the ERM objective. Such an iterative 2-phase algorithm has empirical benefits, as we will show in Suppl.~\ref{sm:metrics-phase2}. Additionally, we propose an ideal object-based matching algorithm when objects are observed.  
 

\textbf{Other work.} Others approaches to domain generalization include meta-learning~\cite{li2018learningmetadg,balaji2018metareg}, dataset augmentation~\cite{volpi2018generalizingunseen,shankar2018generalizing}, parameter decomposition~\cite{piratla2020efficient, li2017pacs}, and enforcing domain-invariance of the optimal
$P(Y|\vecphi(\vecx))$~\cite{arjovsky2019irm,ahuja2020irmgames}. We empirically compare our algorithm to some of them. 

\section{Insufficiency of class-conditional invariance}

\begin{figure*}[th!]
\centering
\subfigure[][Simple Example]{
    \includegraphics[width=0.20\linewidth]{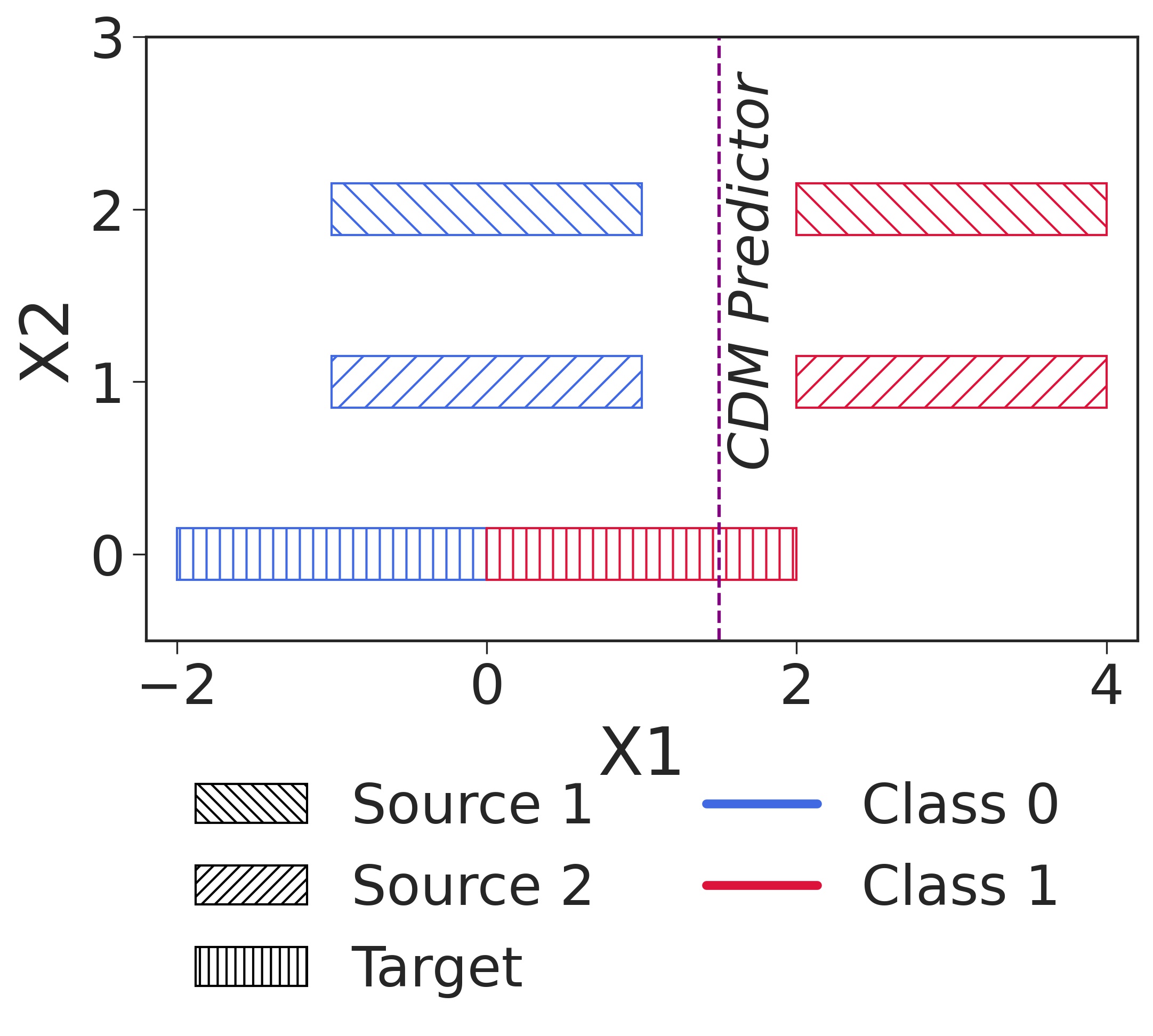}
    \label{fig:example-data}
 }
  \subfigure[][Slab Dataset (\textit{Slab (y-axis) is the stable feature)}]{
    \includegraphics[width=0.68\linewidth]{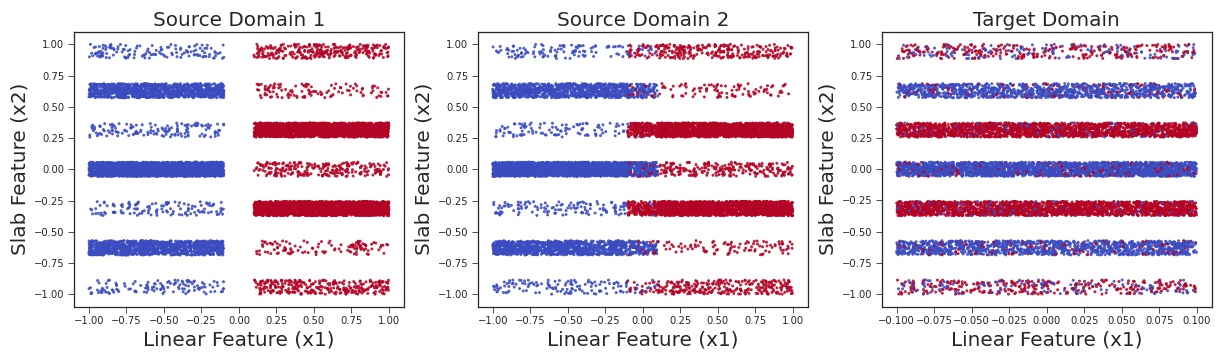}
    \label{fig:slab-data}
 }
 \caption{Two datasets showing the limitations of class-conditional domain-invariance objective. a) The CDM predictor is domain-invariant given the class label but does not generalize to the target domain; b) Colors denote the two ground-truth class labels. For class prediction, the linear feature exhibits varying level of noise across domains. The stable slab feature also has noise but it is invariant across domains.}
 \vspace{-0.2cm}
\label{fig:synthetic-data}
\end{figure*}

Consider a classification task where the learning algorithm has access to i.i.d. data from $m$ domains,  $\{(d_i, \vecx_i, y_i)\}^{n}_{i=1} \sim ({D}_m, \mathcal{X}, \mathcal{Y})^{n}$ where $d_i \in D_m$ and 
$D_m \subset \mathcal{D}$ is a set of $m$ domains. Each training input $(d, \vecx,y)$ is sampled from an unknown  distribution $\mathcal{P}_m(D, X, Y)$. 
  The domain generalization task is to learn a single classifier that generalizes well to  unseen domains $d'\not \in D_m$ and to new data  from the same domains~\cite{shankar2018generalizing}. The optimum classifier can be written as: 
$f^* = \arg \min_{f\in\mathcal{F}} \mathbb{E}_{(d, \vecx, y) \sim\mathcal{P}}[l(y^{(d)}, f(\vecx^{(d)}))]
$
, where $(d, \vecx, y)  \sim \mathcal{P}$ over $(\mathcal{D}, \mathcal{X}, \mathcal{Y})$.


As mentioned above, a popular line of work enforces that the learnt representation $\Phi(\vecx)$ be independent of domain conditional on the class~\cite{li2018conddomaingen, li2018adversarialcada, ghifary2016scatter, hu2019discriminant}, $\Phi(\vecx) \indep D| Y$. 
Below we present two counter-examples showing that the class-conditional objective is not sufficient. 

\subsection{A simple counter-example}
\label{sec:simple-syn}
We construct an example where $\Phi(\vecx) \indep D|Y $, but still the classifier does not generalize to new domains. Consider a two dimensional problem where $x_1 = x_c + \alpha_d; x_2=\alpha_d$ where $x_c$ and $\alpha_d$ are unobserved variables, and $\alpha_d$  varies with domain (Figure~\ref{fig:example-data}).  The true function depends only on the stable feature $x_c$, $y=f(x_c)=I(x_c \geq 0)$. Suppose there are two training domains with $\alpha_1=1$ for domain 1 and  $\alpha_2 =2$ for domain 2, and the test domain has $\alpha_3=0$ (see Figure~\ref{fig:example-data}).  Suppose further that the conditional distribution of $X_C$ given $Y$ is a uniform distribution that changes across domains: for domain 1, $X_c |Y=1 \sim \mathcal{U}(1,3); X_c|Y=0 \sim \mathcal{U}(-2,0);$ and for domain 2,    $X_c |Y=1 \sim \mathcal{U}(0,2); X_c |Y=0 \sim \mathcal{U}(-3,-1)$. Note that the distributions are picked such that  $\phi(x_1, x_2) = x_1$ satisfies the conditional distribution invariant, $\phi(x) \indep D|Y$. The optimal ERM classifier based on this representation, ($I(x_1\geq 1.5)$ has 100\% train accuracy on both domains.  But for the test domain with $\alpha_d=0; X_c|Y=1 \sim \mathcal{U}(0,2); X_c|Y=0 \sim U(-2,0) $, the classifier fails to generalize. It obtains 62.5\% test accuracy (and 25\% accuracy on the positive class), even though its representation satisfies class-conditional domain invariance. In comparison, the ideal representation is $x_1-x_2$ which attains  100\% train accuracy and 100\% test domain accuracy, and does not  satisfy the class-conditional invariant. 

The above counter-example is due to the changing distribution of $x_c$ across domains. If $P(X_c|Y)$ stays the same across domains, then class-conditional methods would not incorrectly pick $x_1$ as the representation. Following \cite{akuzawa2019adversarial}, we claim the following (proof in Suppl.~\ref{app:proofprop1}).
\begin{restatable}{proposition}{classcondprop}
Under the domain generalization setup as above, if $P(X_c|Y)$ remains the same across domains where $x_c$ is the stable feature, then the class-conditional domain-invariant objective for learning representations yields a generalizable classifier such that the learnt representation $\Phi(\vecx)$ is independent of the domain given $x_c$. Specifically, the entropy $H(d|x_c) = H(d|\Phi, x_c)$.
\end{restatable}

However, if $P(X_C|Y)$ changes across domains, then we cannot guarantee the same: 
 $H(d| x_c)$ and $H(d|\Phi, x_c)$ may not be equal. For building generalizable classifiers in such cases,  this example shows that we need an additional constraint on $\Phi$, $H(d| x_c)=H(d|\Phi, x_c)$; i.e. domain and representation should be independent conditioned on $x_c$. 
 

\subsection{An empirical study of class-conditional methods}
\label{sec:slab-dataset}

As a more realistic example, 
consider the slab dataset introduced for detecting simplicity bias in neural networks~\cite{shah2020pitfalls} that contains a feature with spurious correlation. It comprises of two  features and a binary label; ($x_1$) has a linear relationship with the label and the other feature ($x_2$) has a piece-wise linear relationship with the label which is a stable relationship. 
The relationship of the linear feature with the label changes with domains (\ref{app:syn-implement}); we do so by adding noise with probability $\eps=0$ for domain 1 and $\eps=0.1$ for domain 2. On the third (test) domain, we add noise with probability 1 (see Figure~\ref{fig:slab-data}).  We expect that methods that rely on the spurious feature $x_1$ would not be able to perform well on the out-of-domain data. 

The results in Table~\ref{tbl:slab-results} (implementation details in Appendix \ref{app:syn-implement}) show that ERM is unable to learn the slab feature, as evident by poor generalization to the target domain, despite very good performance on the source domains. We also show that  methods based on learning invariant representations by unconditional (DANN, MMD, CORAL) and conditional distribution matching (CDANN, C-MMD, C-CORAL), and matching same-class inputs (RandomMatch)~\cite{motiian2017ccsa} fail to learn the stable slab feature. Note that Proposition 1 suggested the failure of conditional distribution matching (CDM) algorithms when the distribution of stable feature (slab feature) is different across the source domains. However, the slab dataset has similar distribution of stable feature (slabs) across the source domains, yet the CDM algorithms fail to generalize to the target domain. It can be explained by considering the spurious linear feature, which can also satisfy the CDM constraint by ``shifting'' the $y$-conditional distributions along the linear feature. We conjecture that the model may first learn the linear feature due to its simplicity~\cite{shah2020pitfalls}, and then retain the spurious linear feature upon further optimization since it satisfies the CDM constraint. This shows that the CDM methods can empirically fail even when there is an equal distribution of stable features across domains. 

How can we ensure that a model learns the stable, generalizable feature $x_2$? We turn to our example above, where the required invariant was that the representation $\Phi(\vecx)$ should be independent of domain given the stable feature. We apply this intuition and construct a model that enforces that the learnt representation be independent of domain given $x_2$. We do so by minimizing the $\ell_2$-norm of the representations for data points from different domains that share the same slab value  (details of the \textit{PerfectMatch} method in Section~\ref{sm:perf-match-approach}). The results improve substantially: out-of-domain accuracy is now 78\%.  



In the next section, we formalize the intuition of conditioning on stable features $x_c$ using a causal graph, and introduce the concept of \textit{objects} that act as proxies of stable features. 

\begin{table}[t]
\caption{Slab Dataset: Source domains with noisy linear component with probability 0.0 and 0.1, target domain with noise 1.0. Mean and standard deviation over 10 different seed values for each method. The results for DANN~\cite{ganin2016dann}, CDANN~\cite{li2018adversarialcada}, MMD, C-MMD~\cite{li2018adversarialfeature}, CORAL, C-CORAL~\cite{sun2016deep} were computed by using their implementations in DomainBed~\cite{gulrajani2020domainbed}.}
\centering
\footnotesize
\begin{tabular}{@{}l | c c |c@{}}
\toprule
Method & Source 1 & Source 2 & Target \\ 
\midrule
ERM & \textbf{100.0} (0.0 ) & 96.0 (0.25) & 57.6 (6.58) \\
DANN & 99.9 (0.07) & 94.8 (0.25) & 53.0 (1.41)\\
MMD & 99.9 (0.01) & 95.9 (0.27) & 62.9 (5.01) \\
CORAL &  99.9 (0.01) & 96.0 (0.27) & 63.1 (5.86) \\
\midrule
RandMatch & \textbf{100.0} (0.0) & 96.1 (0.22) & 59.5 (3.50) \\
CDANN & 99.9 (0.01) & 96.0 (0.27) & 55.9 (2.47)\\
C-MMD & 99.9 (0.01) &  96.0 (0.27) & 58.9 (3.43) \\
C-CORAL & 99.9 (0.01) & 96.0 (0.27) & 64.7 (4.69) \\
\midrule
PerfMatch & 99.9 (0.05) & \textbf{97.8} (0.28)  & \textbf{77.8} (6.01) \\
\bottomrule
\end{tabular}
\label{tbl:slab-results}
\end{table}

\section{A Causal View of Domain Generalization}
\label{sec:def-dg}
\subsection{Data-generating process}
\label{sec:dgp}
\figref{fig:scm-mnist} shows a structural causal model (SCM) that describes the data-generating process for the domain generalization task. For intuition, consider a  task of classifying the type of item or screening an image for a medical condition. Due to human variability or by design (using data augmentation), the data generation process yields variety of images for each class, sometimes  multiple views for the \emph{same object}. Here each view can be considered as a different \emph{domain} $D$, the label for item type or medical condition as the class $Y$, and the image pixels as the features $X$. Photos of the same item or the same person correspond to a common \emph{object} variable, denoted by $O$. To create an image, the data-generating process first samples an object and view (domain) that may be correlated to each other 
(shown with dashed arrows). The pixels in the photo are caused by both the object and the view, as shown by the two incoming arrows to $X$. The object also corresponds to  high-level \textit{causal} features $X_C$ that are common to any image of the same object, which in turn are used by humans to label the class $Y$. We call $X_C$ as causal features because they directly cause the class  $Y$.
\begin{figure}[t]
\centering
\subfigure[][Image classification.]{
   \includegraphics[width=0.38\linewidth]{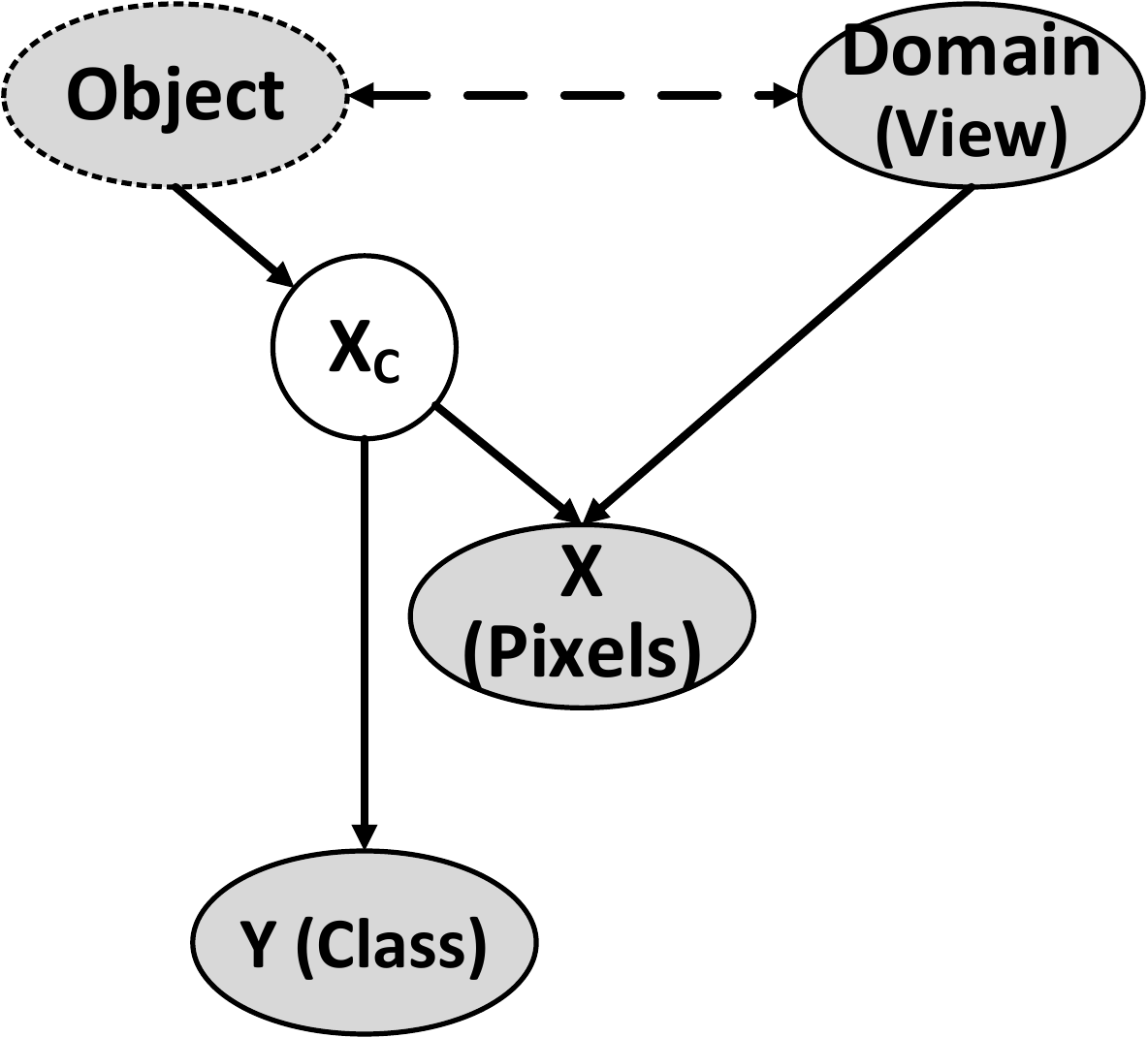}
   \label{fig:scm-mnist}
 }
 \subfigure[][General SCM.]{
    \includegraphics[width=0.38\linewidth]{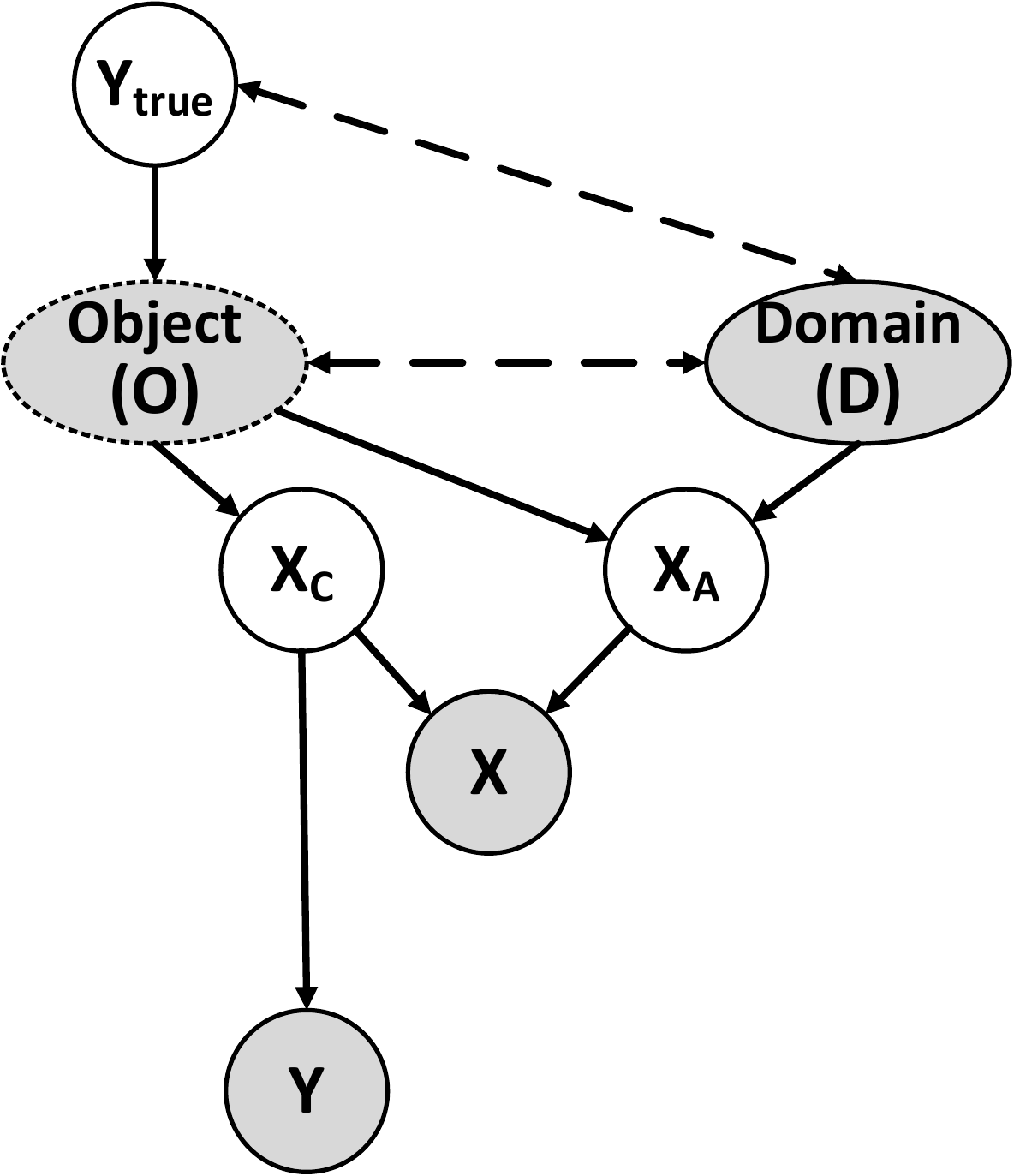}
    \label{fig:scm-withnoise}
 }
 \caption{Structural causal models for the data-generating process. Observed variables are shaded; dashed arrows denote correlated nodes. 
    \emph{Object}  may not be observed.}
\label{fig:scm-all}
\end{figure}

The above example is typical of a domain generalization problem; a general SCM is shown in \figref{fig:scm-withnoise}, similar to the graph in \cite{heinze2017conditional}.  In general, the underlying \emph{object} for each input $\vecx_i^{(d)}$  may not be observed. Analogous to the object-dependent (\textit{causal}) features $X_C$, we introduce a node for domain-dependent high-level features of the object $X_A$. Changing the domain can be seen as an intervention: for each observed $\vecx_i^{(d)}$, there are a set of (possibly unobserved) counterfactual inputs $\vecx_j^{(d')}$ where $d \neq d'$, such that all correspond to the same object (and thus share the same $X_C$). 
For completeness, we also show the true unobserved label of the object which led to its generation as $Y_{true}$ (additional motivation for the causal graph is in Suppl.~\ref{app:graphdetails}). Like the object $O$, $Y$ may be correlated with the domain $D$.  
Extending the model in \cite{heinze2017conditional}, we  allow that objects can be correlated with the domain conditioned on $Y_{true}$. 
As we shall see, considering the relationship of the \emph{object} node becomes the key piece for developing the invariant condition. 
The SCM corresponds to the following non-parametric equations. 
{\small
 \begin{align*}
     o &:= g_o (y_{true}, \eps_o, \eps_{od}) &  & \vecx_c = g_{xc}(o)  \\
     \vecx_a &:= g_{xa}(d,o, \eps_{xa}) & \vecx := g_x(\vecx_c, \vecx_a, \eps_x) & y := h(\vecx_c, \eps_y) 
\end{align*}
}%
where $g_o$, $g_{xc}$, $g_{xa}$, $g_{x}$ and $h$ are general non-parametric functions. The error $\eps_{od}$ is correlated with domain $d$ whereas $\eps_o$, $\eps_{xa}$, $\eps_x$ and $\eps_y$  are mutually independent error terms that are independent of all other variables. Thus, noise in the class label is independent of domain. Since $x_c$ is common to all inputs of the same object, $g_{xc}$ is a deterministic function of $o$. 
In addition, the SCM provides conditional-independence conditions that all data distributions $\mathcal{P}$ must satisfy, through the concept of d-separation (Suppl.~\ref{app:d-sep}) and the perfect map assumption~\cite{pearl2009book}. 



\subsection{Identifying the invariance condition}
\label{sec:objcond}
From Figure~\ref{fig:scm-withnoise}, $X_C$ is the node that causes $Y$. Further, by d-separation, the class label is independent of domain conditioned on $X_C$, $Y\indep D|X_C$. Thus our goal is to learn $y$ as $h(\vecx_c)$ where $h: \mathcal{C} \rightarrow \mathcal{Y}$. The ideal loss-minimizing function $f^*$ can be rewritten as (assuming $\vecxc$ is known):
\begin{equation}\label{eq:rep-ideal-loss}
\small
\arg \min_f \mathbb{E}_{(d, \vecx,y)}l(y, f(\vecx))= \arg \min_h \mathbb{E}[l(y, h(\vecxc))]
\end{equation}
Since $X_C$ is unobserved, this implies that we need to learn it through a representation function $\vecphi:\mathcal{X} \rightarrow \mathcal{C}$.  
Together, $h(\vecphi(x))$ leads to the desired classifer $f:\mathcal{X} \rightarrow \mathcal{Y}$. 

\noindent \textbf{Negative result on identification}.
Identification of causal features is a non-trivial problem~\cite{magliacane2018domainadaptcausal}. We first show that $x_C$ is unidentifiable given observed data $P(X, Y, D, O)$ over multiple domains.
Given the same probability distribution $P(X,Y,D,O)$, multiple values of $X_C$ are possible. Substituting for $o$ in the SCM equations, we obtain, $y=h(g_{xc}(o), \eps_y); \vecx=g_x(g_{xc}(o), g_{xa}(d,o,\eps_{xa}), \eps_x)$. By choosing $g_x$ and $h$ appropriately, different values of $g_{xc}$ (that determine $x_c$ from $o$) can lead to the same observed values for $(y,d,o,x)$.  
The proof for the following proposition is in Supp.~\ref{app:proofprop2}.
\begin{restatable}{proposition}{identifyprop}
Given observed data distribution $P(Y,X, D, O)$ that may also include data obtained from interventions on domain $D$, multiple values of $X_C$ yield exactly the same observational and interventional distributions and hence $X_c$ is unidentifiable.
\end{restatable}

\subsection{A ``perfect-match'' invariant}
\label{sm:perf-match-approach}
In the absence of identifiability, we proceed to find an invariant that can characterize $X_c$. 
By the d-separation criterion, we see that $X_C$ satisfies two conditions: \textbf{ 1)} $X_C \indep D| O$, {\bf 2)}
    $X_C \not \indep O$; where $O$ refers to the object variable and $D$ refers to a domain.
The first is an invariance condition: $X_C$ does not change with different domains for the same object. 
   To enforce this, we stipulate that the average pairwise distance between $\vecphi(x)$ for inputs across  domains for the same object is 0, 
$\sum_{\Omega(j,k)=1; d\neq d'} \operatorname{dist}(\vecphi(\vecx_j^{(d)}), \vecphi(\vecx_k^{(d')})) = 0$.
Here $\Omega:\mathcal{X} \times \mathcal{X}\rightarrow \{0,1\}$  is a \emph{matching} function that is 1 for pairs of inputs across domains corresponding to the same object, and  0 otherwise.

However, just the above invariance will not work: we need the representation to be informative of the object $O$ (otherwise even a constant $\vecphi$ minimizes the above loss).  
Therefore, the second condition stipulates that $X_C$ should be informative of the object, and hence about $Y$. 
We add the standard classification loss, leading to constrained optimization,
\begin{multline} \label{eq:erm-pmatch-hardloss}
\begin{medsize}
    f_{\tt perfectmatch} = \arg \min_{h, \vecphi} \sum_{d=1}^{m} L_d(h(\vecphi(X)), Y) \end{medsize}\\ \begin{medsize}\texttt{ s.t.} \sum_{\Omega(j,k)=1; d\neq d'} \operatorname{dist}(\vecphi(\vecx_j^{(d)}), \vecphi(\vecx_k^{(d')})) = 0 \end{medsize}
    \end{multline}
where $L_d(h(\vecphi(X), Y))=\sum_{i=1}^{n_d} l(h(\vecphi(\vecx_i^{(d)}), y_i^{(d)}) $. 
Here $f$ represents the composition $h \circ \vecphi$. E.g., a neural network with $\vecphi(x)$  as its r${th}$ layer, and $h$ being the rest of the layers. 

Note that there can be multiple $\Phi(\vecx)$ (e.g., linear transformations) that are equally good for the prediction task. Since $x_c$ is unidentifiable, we focus on the set of \textit{stable} representations that are d-separated from $D$ given $O$. Being independent of domain given the object, they cannot have any association with $X_a$, the high-level features that directly depend on domain (Figure~\ref{fig:scm-all}b).
The proof for the next theorem is in Suppl.~\ref{app:thm1proof}. 

\begin{restatable}{theorem}{perfectmatchtheorem}\label{thm:main-match}
    For a finite number of domains $m$, as the number of examples in each domain $n_d\rightarrow \infty$, \\
        \noindent 1. The set of representations that satisfy the  condition  $\sum_{\Omega(j,k)=1; d\neq d'} \operatorname{dist}(\vecphi(\vecx_j^{(d)}), \vecphi(\vecx_k^{(d')})) =0$ contains the optimal $\vecphi(\vecx)=X_C$ that minimizes the domain generalization loss in \eqref{eq:rep-ideal-loss}. \\
    \noindent 2. Assuming that $P(X_a|O, D)<1$ for every high-level feature $X_a$ that is directly caused by domain, and for P-admissible loss functions~\cite{miller1993loss} whose minimization is conditional expectation (e.g., $\ell_2$ or cross-entropy), 
a loss-minimizing classifier for the  following loss  is the true function $f^*$, 
    for some value of $\lambda$. 
    \vspace{-0.2em}
\begin{multline} \label{eq:erm-pmatch-loss}
    \begin{medsize}
    f_{\tt perfectmatch} = \arg \min_{h, \vecphi} \sum_{d=1}^{m} L_d(h(\vecphi(X)), Y) +\end{medsize} \\ \begin{medsize}\lambda \sum_{\Omega(j,k)=1; d\neq d'} \operatorname{dist}(\vecphi(\vecx_j^{(d)}), \vecphi(\vecx_k^{(d')})) \end{medsize}
\end{multline}
\end{restatable}


\subsection{Past work: Learning common representation}
\label{sec:limit-prior}
Using the SCM, we now compare the proposed invariance condition to  domain-invariant and class-conditional domain-invariant objectives. 
d-separation results show that both these objectives are incorrect: in particular, the class-conditional objective $\vecphi(\vecx) \indep D|Y$ is not satisfied by $X_C$,  ($X_C \not \indep D|Y_{true}$) due to a path through $O$. Even with infinite data across domains, they will not learn the true $X_C$. The proof is in Suppl.~\ref{app:proofcorr1}.

\begin{restatable}{proposition}{othermethodscor}
The conditions enforced by domain-invariant ($\vecphi(x)\indep D$) or class-conditional domain-invariant ($\vecphi(x) \indep D | Y$) methods are not satisfied by the causal representation $X_C$. Thus, without additional assumptions, the set of representations that satisfy any of these conditions does not contain $X_C$, even as $n\rightarrow \infty$. 
\end{restatable}

\section{MatchDG: Matching without objects}
When object information is available, Eq. \eqref{eq:erm-pmatch-loss} provides a loss objective to build a classifer using causal features. However, object information is not always available,  
and in many datasets there may not be a
perfect ``counterfactual'' match based on same object across domains.  
Therefore, we propose a  two-phase, iterative contrastive learning method to approximate object matches. 

The object-invariant condition from Section~\ref{sec:objcond} can be interpreted as matching pairs of inputs from different domains that share the same $X_C$. To approximate it, our goal is to learn a matching $\Omega:\mathcal{X} \times\mathcal{X}\rightarrow \{0,1\}$ such that pairs having  $\Omega(\vecx, \vecx')=1$ have low difference in $\vecxc$ and $\vecx'_c$. We make the following assumption.

\begin{assumption}
Let $(\vecx_i^{(d)}, y)$, $(\vecx_j^{(d')},y)$ be any two points that belong to the same class, and let  $(\vecx_k^{(d)} y')$ be any other point that has a different class label. Then the distance in causal features  between $\vecx_i$ and $\vecx_j$ is smaller than that between $\vecx_i$ and $\vecx_k$ or $\vecx_j$ and $\vecx_k$: $\operatorname{dist}(x_{c,i}^{(d)},
    x_{c,j}^{(d')}) \leq \operatorname{dist}(x_{c,i}^{(d)},
    x_{c,k}^{(d')})$ and $\operatorname{dist}(x_{c,j}^{(d)},
    x_{c,i}^{(d')}) \leq \operatorname{dist}(x_{c,j}^{(d)},
    x_{c,k}^{(d')})$. 
\end{assumption}

\subsection{Two-phase method with iterative matches}
\label{sec:two-phase}

To learn a  matching function $\Omega$, we use unsupervised contrastive learning from ~\cite{chen2020simclr,he2019moco} and adapt it to construct an iterative \mdg\ algorithm that updates the both the representation and matches after each epoch. The algorithm relies on the property that two inputs from the same class have more similar causal features than inputs from different classes. 

\xhdr{Contrastive Loss.} To find matches, we optimize a contrastive representation learning loss that minimizes distance between  same-class  inputs from different domains in comparison to inputs from different classes across  domains. 
Adapting the contrastive loss for a single domain~\cite{chen2020simclr}, we consider \emph{positive} matches as two inputs with the same class but different domains, and \emph{negative} matches as pairs with  different classes. 
For every positive match pair $(\vecx_j, \vecx_k)$, we propose a loss where $\tau$ is a hyperparameter, $B$ is the batch size, and 
$\operatorname{sim}(\mathbf{a},\mathbf{b})=\vecphi(\vecx_a)^T \vecphi(\vecx_b) /\norm{\vecphi(\vecx_a)} \norm{\vecphi(\vecx_b)}$  
is the cosine similarity.
\begin{equation} \label{eq:ntloss}
  \small
  l(\vecx_j,\vecx_k) = -\log \frac{e^{\operatorname{sim}(j, k)/\tau}}{ e^{\operatorname{sim}(j, k)/\tau} + \sum_{i=0, y_i\neq y_j}^B e^{\operatorname{sim}(j, i)/\tau}}
\end{equation}

\xhdr{Iterative matching.} Our key insight is to update the positive matches during training. We start training with a random set of positive matches based on the classes, but after every $t$ epochs, we update the positive matches based on the nearest same-class pairs in representation space and iterate until convergence. Hence for each anchor point, starting with an initial set of positive matches, in each epoch a representation is learnt using contrastive learning; after which the positive matches are themselves updated based on the closest same-class data points across domains in the representation. As a result, the method differentiates between data points of the same class instead of treating all of them as a single unit. With iterative updates to the positive matches, the aim is to account for intra-class variance across domains and match data points across domains that are more likely to share the same base object. In Suppl.~\ref{sm:iterative-ctr}, we compare the gains due to the proposed iterative matching versus standard contrastive training.

Obtaining the final representation completes Phase I of the algorithm. 
In Phase II, we use this representation to compute a new match function based on closest same-class pairs and apply Eq. \eqref{eq:erm-pmatch-loss} to obtain a classifier regularized on those matches.


\begin{algorithm}[t]
\footnotesize
\begin{algorithmic}
	\STATE {{\bf In:} Dataset $(d_i,x_i,y_i)^n_{i=1}$ from m domains, $\tau$, t}
	\STATE{{\bf Out:} Function $f:\mathcal{X} \to \mathcal{Y}$}

	\STATE{Create random match pairs $\Omega_Y$.} 
	\STATE{Build a $p*q$ data matrix $\mathcal{M}$.} 
	
    \STATE{ \textbf{Phase I}}
  	\WHILE{notconverged} 
		\FOR{$batch \sim \mathcal{M}$}
   			\STATE {Minimize contrastive loss \eqref{eq:ntloss}.}
		\ENDFOR
		\IF{epoch \% t == 0}
		 	\STATE{Update match pairs using $\vecphi_{epoch}$.}
		\ENDIF
   \ENDWHILE
   
    \STATE{ \textbf{Phase II}}
    \STATE{Compute matching based on $\vecphi$.}
    \STATE{Minimize the loss \eqref{eq:erm-pmatch-loss} with learnt match function $\vecphi$ to obtain $f$.}
\end{algorithmic}
\caption{MatchDG}
\label{alg:matchdg}
\end{algorithm}

\xhdr{The importance of using two phases.} We implement \mdg\ as a 2-phase method, unlike previous methods~\cite{motiian2017ccsa,dou2019masf} that employed class-based contrastive loss as a regularizer with ERM. This is to avoid the classification loss interfering with the goal of learning an invariant representation across domains (e.g., in datasets where one of the domains has many more samples than others). Therefore, we first learn the match function using only the contrastive loss. 
Our results in Suppl.~\ref{sm:metrics-phase2} show that the two-phase method provides better overlap with ground-truth perfect matches than optimizing classification and matching simultaneously. 


To implement \mdg\, we build a $p \times q$ data matrix containing $q-1$ positive matches for each input and then sample mini-batches from this matrix. The last layer of the
contrastive loss network is considered as the learnt representation (see Algorithm~\ref{alg:matchdg}; details are in Suppl.~\ref{app:impl}).

\subsection{MDG Hybrid}
While \mdg\ assumes no information about objects, it can be easily augmented to incorporate information about known objects. For example, in computer vision, a standard practice is to augment data by performing rotations, horizontal flips, color jitter, etc. These self-augmentations provide us with access to known objects, which can included as perfect-matches in \mdg\ Phase-II by adding another regularizer to the loss from Eq \ref{eq:erm-pmatch-loss}. We name this method \hybrid\ and evaluate it alongside \mdg\ for datasets where we can perform self augmentations.    

\begin{table*}[t]
\caption{Accuracy for Rotated MNIST \& Fashion-MNIST datasets on target domains of $0^{\circ}$ and $90^{\circ}$. Accuracy for CSD~\cite{piratla2020efficient}, MASF~\cite{dou2019masf}, IRM~\cite{arjovsky2019irm} are reproduced from their code. Results for the other versions of Rotated MNIST with all test angles (LetNet~\cite{motiian2017ccsa}, DomainBed~\cite{gulrajani2020domainbed}) are in Suppl.~\ref{sm:2-layer}, \ref{sm:mnist-domain-bed}.}
\centering
\footnotesize

\begin{tabular}{@{}l | p{1.3cm} |c@{}| c @{}|c@{}| c@{}| c@{}| c @{}|c @{}}
\toprule
Dataset                                                                            & \begin{tabular}[c]{@{}l@{}}Source \end{tabular}      & ERM  & MASF & CSD & IRM & \begin{tabular}[c]{@{}l@{}}\texttt{RandMatch} \end{tabular}  & {\begin{tabular}[c]{@{}l@{}}\mdg \end{tabular}} & \begin{tabular}[c]{@{}l@{}}\texttt{PerfMatch (Oracle)} \end{tabular} \\ \midrule
\multirow{3}{*}{\begin{tabular}[c]{@{}l@{}}Rotated\\ MNIST\end{tabular}}           & \begin{tabular}[c]{@{}l@{}}15, 30, 45, \\ 60, 75\end{tabular} & 93.0 (0.11)  & 93.2  (0.2)  & 94.5  (0.35) & 92.8 (0.53) & 93.4 (0.26)   & {\bf 95.1} (0.25) & 96.0 (0.41)  \\ \cmidrule {2-8}
                                                                                 & 30, 45, 60                                                    & 76.2 (1.27)  & 69.4 (1.32) & 77.7 (1.88) & 75.7 (1.11) & 78.3 (0.55)  &  { \bf 83.6 } (1.44) &  89.7 (1.68)   \\ \cmidrule {2-8}
                                                                                   & 30, 45                                                        & 59.7 (1.75) &  60.8 (1.53)   &  62.0 (1.31) & 59.5 (2.61)  & 63.8 (3.92)     & {\bf 69.7} (1.30) & 80.4 (1.79)  \\ 
                                                                                    \midrule 
\multirow{3}{*}{\begin{tabular}[c]{@{}l@{}}Rotated\\ Fashion\\ MNIST\end{tabular}} 

& \begin{tabular}[c]{@{}l@{}}15, 30, 45, \\  60, 75\end{tabular} 
&  77.9 (0.13)      &  72.4  (2.9)    & 78.7 (0.38) &  77.8  (0.02) & 77.0 (0.42) & \textbf{80.9} (0.26)  & 81.6 (0.46)   \\ \cmidrule {2-8}
                                                                                & 30, 45, 60                                                     & 36.1 (1.91)  & 29.7 (1.73)      & 36.3 (2.65) & 37.8 (1.85) & 38.4 (2.73) & {\bf 43.8} (1.33) & 54.0 (2.79)   \\ \cmidrule {2-8}
                                                                                   & 30, 45                                                        & 26.1 (1.10)     &  22.8 (1.26)   &   24.2 (1.69) & 26.6 (1.06)  & 26.9 (0.34)          & {\bf 33.0} (0.72) &    41.8 (1.78)   \\ \bottomrule 
\end{tabular}
\label{tbl:mnist}
\end{table*}

\begin{table}[th!]
\caption{Overlap with perfect matches. top-10 overlap and the mean rank for perfect matches for \mdg\ and ERM over all training domains. Lower is better for mean rank.}
\centering
\footnotesize
 \resizebox{\linewidth}{!}{
\begin{tabular}{@{}l | l |c c c @{}}
\toprule
Dataset                                                                            & \begin{tabular}[c]{@{}c@{}}Method \end{tabular} & \begin{tabular}[c]{@{}c@{}}Overlap (\%)\end{tabular}  &  \begin{tabular}[c]{@{}c@{}}Top 10 \\Overlap (\%) \end{tabular} &  Mean Rank
\\ \midrule
\multirow{3}{*}{\begin{tabular}[c]{@{}l@{}} MNIST\end{tabular}}           & \begin{tabular}[c]{@{}l@{}}
                                        ERM\end{tabular}  & 15.8 (0.42)  & 48.8 (0.78)  & 27.4 (0.89)   \\ 
                                                                                &\begin{tabular}[c]{@{}c@{}}\mdg \\ (Default) \end{tabular}                                                    & {\bf 28.9} (1.24) & {\bf 64.2} (2.42)  & {\bf 18.6} (1.59) \\
                                                                                   & \begin{tabular}[l]{@{}l@{}}\mdg \\ (PerfMatch) \end{tabular}                                                         & 47.4 (2.25)  & 83.8 (1.46)  & 6.2 (0.61)   \\ 
                                                                                    \midrule
\multirow{3}{*}{\begin{tabular}[c]{@{}l@{}}Fashion\\ MNIST\end{tabular}}           & \begin{tabular}[c]{@{}l@{}}
                                         ERM\end{tabular}  & 2.1 (0.12)  & 11.1 (0.63)   & 224.3 (8.73)  \\ 
                                                                                   
                                                                                              & \begin{tabular}[c]{@{}c@{}}\mdg \\ (Default) \end{tabular}                                                      & {\bf 17.9} (0.62) &  {\bf 43.1} (0.83) & {\bf 89.0} (3.15)   \\\cmidrule {2-5}
                                                                                   & \begin{tabular}[l]{@{}l@{}}\mdg \\ (PerfMatch) \end{tabular}                                                         & 56.2 (1.79)  & 87.2 (1.48)  &  7.3 (1.18)    \\                              \bottomrule 
\end{tabular}}
\label{tbl:metrics}
\end{table}

\section{Evaluation}
\label{sec:eval}
We evaluate out-of-domain accuracy of \mdg\ on  two simulated benchmarks by \citet{piratla2020efficient}, Rotated MNIST and Fashion-MNIST, on PACS dataset \cite{li2017pacs}, and  on a novel Chest X-rays dataset. In addition, using the simulated datasets, we inspect the quality of matches learnt by \mdg\ by comparing them to ground-truth object-based matches. For PACS and Chest X-rays, we also implement \hybrid\ that uses augmentations commonly done while training neural networks. We compare to 
1) ERM: Standard empirical risk minimization,
2) \ermatch\ that implements the loss from Eq. \eqref{eq:erm-pmatch-loss} but with randomly selected matches from the same class~\cite{motiian2017ccsa}, 
3) other state-of-the-art methods for each dataset. 
For all matching-based methods, we use the cross-entropy loss for $L_d$ and $\ell_2$ distance for $\operatorname{dist}$ in Eq.\eqref{eq:erm-pmatch-loss}. Details of implementation and the datasets are in Suppl.~\ref{app:impl}.  All the numbers are averaged over $3$ runs with standard deviation in brackets.

\textbf{Rotated MNIST \& Fashion-MNIST.} The datasets contain rotations of  grayscale MNIST handwritten digits and fashion article images from $0^{\circ}$ to $90^{\circ}$ with an interval of $15^{\circ}$~\cite{ghifary2015multitaskae}, where each rotation angle represents a domain and the  task is to predict the class label. Since  different domains' images are generated from the same base image (object), there exist perfect matches across domains. Following CSD, we report accuracy on $0^{\circ}$ and $90^{\circ}$ together as the test domain and the rest as the train domains; since these test angles, being extreme, are the hardest to generalize to (standard setting results are in  Suppl.~\ref{sm:2-layer}, \ref{sm:mnist-domain-bed}).  

\textbf{PACS.} This dataset contains total  $9991$ images from four domains: Photos (P), Art painting (A), Cartoon (C) and Sketch (S). The task is to classify objects over $7$ classes. Following~\cite{ dou2019masf}, we train $4$ models with each domain as the target  using Resnet-18, Resnet-50 and Alexnet. 

\textbf{Chest X-rays.} We introduce a harder real-world dataset based on Chest X-ray images from three different sources: NIH~\cite{wang2017chestx}, ChexPert~\cite{irvin2019chexpert} and RSNA~\cite{rsna}. The task is to detect whether the image corresponds to a patient  with Pneumonia (1) or not (0). To create spurious correlation, all images of class 0 in the training domains are translated vertically downwards; while no such translation is done for the test domain.

\xhdr{Model Selection.} While using a validation set from the test domain may improve classification accuracy, it goes  against the problem motivation of generalization to unseen domains. Hence, we use only data from source domains to construct a validation set (except when explicitly mentioned in Table 4, to compare to past methods that use test domain validation). 

\subsection{Rotated MNIST and Fashion MNIST} Table~\ref{tbl:mnist} shows classification accuracy on \rotm\ and \rotfm\ for test domains $0^{\circ}$ \& $90^{\circ}$ using  Resnet-18 model. On both datasets, \mdg\ {\em outperforms} all baselines. The last column shows the accuracy for an oracle method, \epmatch\  that has access to ground-truth perfect matches across domains. \mdg's accuracy lies between \ermatch\ and \epmatch, indicating the benefit of learning a matching function. As the number of training domains decrease, the gap between \mdg\ and baselines is highlighted: with 3 source domains for \rotfm, \mdg\ achieves accuracy of $43.8\%$ whereas the next best method \ermatch\ achieves $38.4\%$.

We also evaluate on a simpler 2-layer LeNet~\cite{motiian2017ccsa}, and the model from  \cite{gulrajani2020domainbed} to compare \mdg\ to  prior works ~\cite{ilse2020diva, ganin2016dann, shankar2018generalizing, goodfellow2014explaining}; the results are in Suppl.~\ref{sm:2-layer}, \ref{sm:mnist-domain-bed}. 
    
\xhdr{Why \mdg\ works?}
We compare the matches returned by \mdg\ Phase I (on Resnet-18 network) to the ground-truth perfect matches and find that it has significantly higher overlap than matching based on ERM loss (Table~\ref{tbl:metrics}). We report three metrics on the representation learnt: 
percentage of \mdg\ matches that are  perfect matches, \%-age of inputs for which the perfect match is within the top-10 ranked \mdg\ matches, and mean rank of perfect matches measured by distance over the \mdg\ representation. 

On all three metrics, \mdg\ finds a representation whose matches are more consistent with ground-truth perfect matches. For both \rotm\ and \rotfm\ datasets, about 50\% of the inputs have their perfect match within top-10 ranked matches based on the representation learnt by \mdg\ Phase I. About 25\% of all matches learnt by \mdg\ are perfect matches. For comparison, we also show metrics for an (oracle) MatchDG method that is initialized with perfect matches: it achieves better overall and Top-10 values. 
Similar results for \mdg\ Phase 2 are in Suppl.~\ref{sm:metrics-phase2}. Mean rank for \rotfm\ may be higher because of the larger sample size $10,000$ per domain; 
metrics for training with $2000$ samples are in  Suppl.~\ref{sm:fmnist-lower}. 
To see how the overlap with perfect matches affects accuracy, we simulate random matches with 25\%, 50\% and 75\% overlap with perfect matches (Suppl. Tbl.~\ref{sm:frac-matches}). Accuracy increases with the  fraction of perfect matches, indicating the importance of  capturing good matches.

\begin{figure}[tb]
\centering
\subfigure[][MatchDG Penalty during training]{
   \includegraphics[scale=0.2]{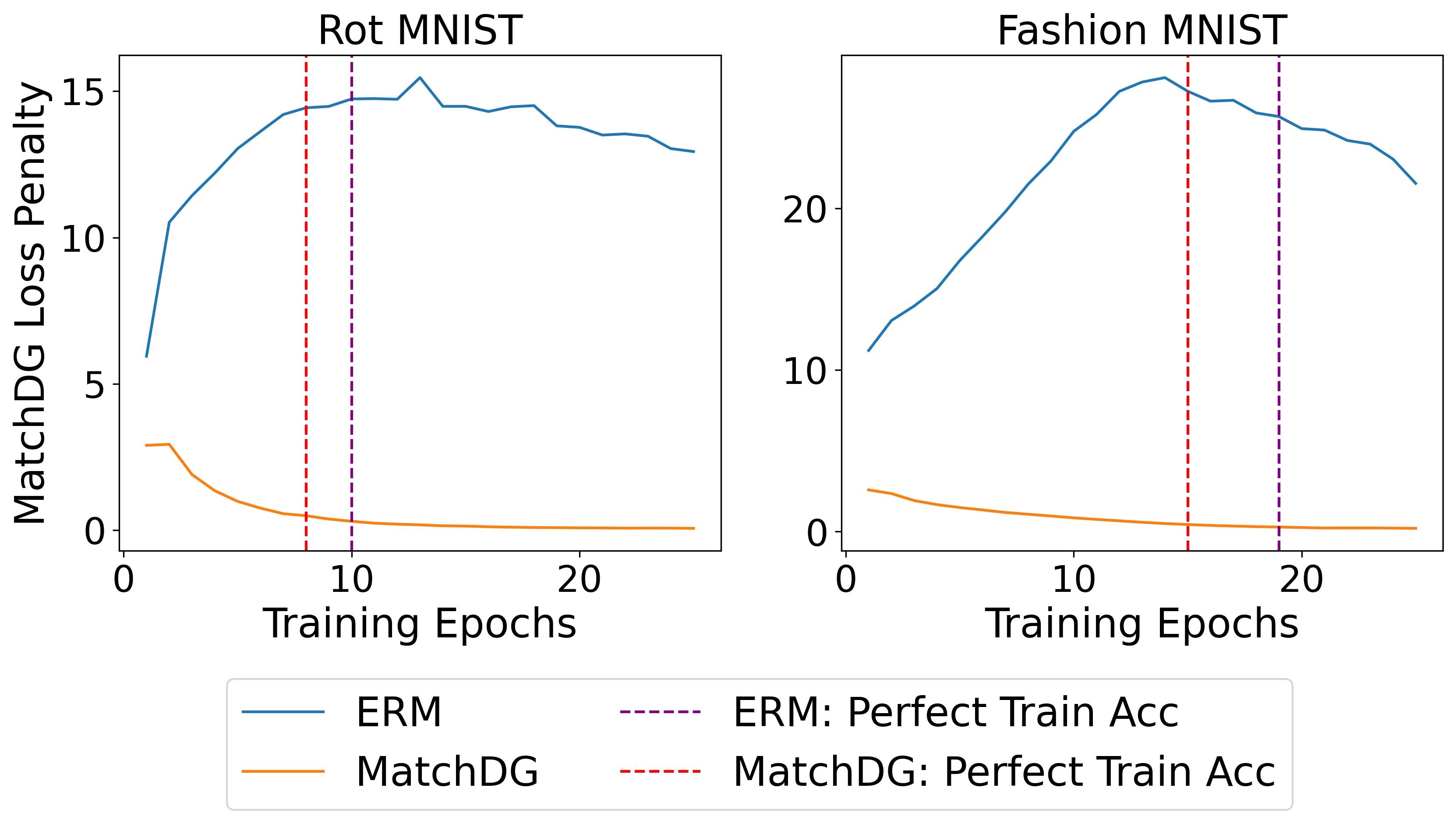} 
   \label{fig:matchdg-plot}
 }
 
\subfigure[][IRM Penalty during training]{
    \includegraphics[scale=0.2]{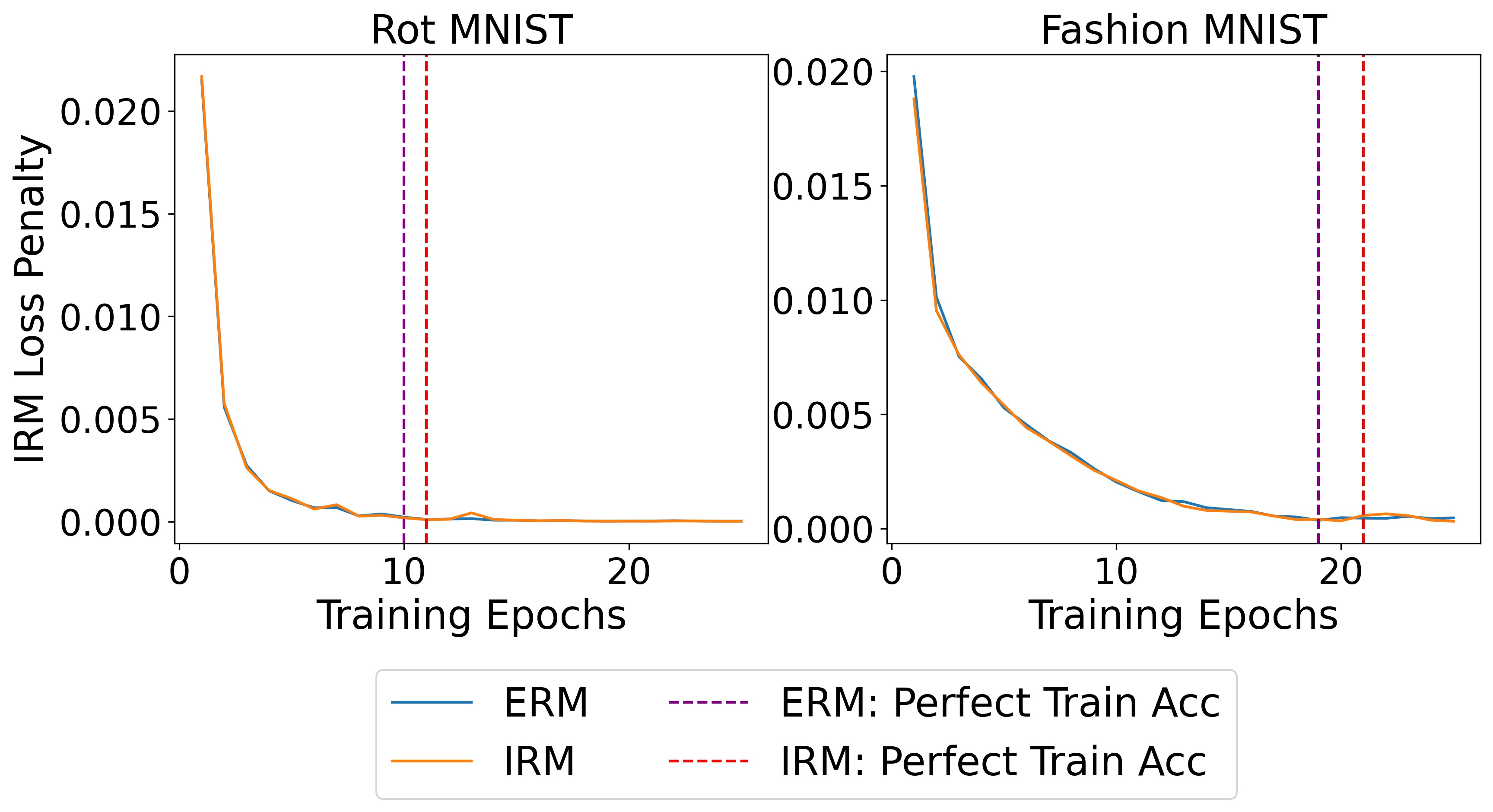}
   \label{fig:irm-plot}
 }
 
 \caption{MatchDG regularization penalty is not trivially minimized even as the training error goes to zero.}
\label{fig:loss-plots}
\end{figure}

\xhdr{MatchDG vs. IRM on zero training error.} Since neural networks often achieve zero training error, we also evaluate the effectiveness of the \mdg\ regularization under this regime. Fig.~\ref{fig:loss-plots} shows the matching loss term as training proceeds for \rotm\ and \rotfm. Even after the model achieves zero training error, we see that plain ERM objective is unable to minimize the matching loss (and thus MatchDG penalty is needed). This is because MatchDG regularization depends on comparing the (last layer) representations, and zero training error does not mean that the  representations  within each class are the same.  In contrast, regularizations that are based on comparing loss between training domains such as the IRM penalty can be satisfied by plain ERM as the training error goes to zero (Fig.~\ref{fig:irm-plot}); similar to Fig. (5) from \cite{krueger2020out} where  ERM can minimize IRM penalty on Colored MNIST.

\begin{table}[tb]
\centering
\caption{Accuracy on PACS with ResNet 18 (default),  and Resnet 18 with test domain validation. The results for 
JiGen~\cite{carlucci2019domain}, 
DDAIG~\cite{zhou2020deep},
SagNet~\cite{nam2019reducing},
DDEC~\cite{asadi2019shapedg}, were taken from the DomainBed~\cite{gulrajani2020domainbed} paper. For G2DM~\cite{albuquerque2020generalizing}, CSD~\cite{piratla2020efficient}, RSC~\cite{huang2020self} it was taken from the respective paper. Extensive comparison with other works and std. dev. in results is in Supp~\ref{sm:resnet}.
}
\label{tbl:pacs-r18}
\footnotesize
\begin{tabular}{@{}l@{}|p{0.9cm} @{}|p{0.9cm}@{}|p{0.9cm}@{}|p{0.9cm}@{}|c@{}}
\toprule
          & P            & A         & C        & S            & Average.  \\ \midrule
ERM       & 95.38 & 77.68  & 78.98 & 74.75 & 81.70 \\
JiGen     & 96.0        & 79.42        & 75.25        & 71.35        & 80.41 \\
G2DM       & 93.75        & 77.78        & 75.54    & 77.58        & 81.16 \\
CSD       & 94.1     & 78.9       & 75.8       & 76.7      & 81.4 
\\
DDAIG     & 95.30     &  \textbf{84.20}        & 78.10         & 74.70         &  83.10 \\  
SagNet     & 95.47      &  83.58        & 77.66         & 76.30         &  83.25 \\  
DDEC       & \textbf{96.93}        & 83.01        & 79.39        & 78.62       & 84.46 \\ 
RSC       & 95.99        & 83.43        & 80.31        & \textbf{80.85}        & \textbf{85.15} \\ 
\texttt{RandMatch} & 95.37 & 78.16  &  78.83  & 75.13 &   81.87 \\
\texttt{MatchDG}   & 95.93  & 79.77  & 80.03  & 77.11 &  83.21 \\ 
\texttt{MDGHybrid}   & 96.15 & 81.71   & \textbf{80.75} & 78.79  & 84.35 \\

\midrule

G2DM    (Test)   & 94.63        & 81.44        & 79.35        & 79.52        & 83.34 \\
\texttt{RandMatch} (Test) & 95.57 & 79.09 &  79.37  & 77.60  &  82.91  \\  
\texttt{MatchDG} (Test)   & 96.53 &  81.32  & 80.70 & 79.72   & 84.56   \\ 
\texttt{MDGHybrid} (Test)   & \textbf{96.67}  & \textbf{82.80}   &  \textbf{81.61} & \textbf{81.05}  & \textbf{85.53} \\ 
\bottomrule
\end{tabular}
\end{table}

\begin{table}[tb]
\centering
\caption{Accuracy on PACS with architecture ResNet 50. The results for IRM~\cite{arjovsky2019irm}, 
CORAL~\cite{sun2016deep},  were taken from the DomainBed~\cite{gulrajani2020domainbed} paper. The result for RSC~\cite{huang2020self} was taken from their paper. Comparison with other works in Supp~\ref{sm:resnet}.
}
\label{tbl:pacs-r50}
\footnotesize
 \resizebox{\linewidth}{!}{
\begin{tabular}{@{}l@{}|p{0.9cm} @{}|p{0.9cm}@{}|p{0.9cm}@{}|p{0.9cm}@{}|c@{}}
\toprule
          & P            & A         & C        & S            & Average.  \\ \midrule
DomainBed    (ResNet50)   & 97.8        & \textbf{88.1}        & 77.9         & 79.1        &  85.7 \\
IRM    (ResNet50)   & 96.7       &  85.0    & 77.6      & 78.5 &  84.4 \\
CORAL    (ResNet50)   & 97.6  &  87.7   & 79.2      & 79.4  &  86.0 \\
RSC    (ResNet50)   & 97.92  &  87.89  & 82.16  & \textbf{83.35}   &  \textbf{87.83} \\

\texttt{RandMatch} (ResNet50) & 97.89 & 82.16 &  81.68  & 80.45 & 85.54  \\
\texttt{MatchDG} (ResNet50)   & 97.94  & 85.61 & 82.12  &  78.76  & 86.11  \\  

\texttt{MDGHybrid} (ResNet50)   & \textbf{98.36} & 86.74  & \textbf{82.32} & 82.66 & 87.52 \\  

\bottomrule
\end{tabular}
}
\end{table}

\subsection{PACS dataset}
\xhdr{ResNet-18.} On the PACS dataset with ResNet-18 architecture (Table~\ref{tbl:pacs-r18}), our methods are competitive to state-of-the-art results averaged over all domains. The  \hybrid\ has the highest average accuracy across domains, except compared to DDEC and RSC. These works do not disclose their model selection strategy (whether the results are using source or test domain validation). Therefore, we also report results of \mdg\ and \hybrid\ using test domain validation, where \hybrid\ obtains comparable results to the best-performing method. In addition, with DDEC~\cite{asadi2019shapedg}, it is not a fair comparison since they use additional style transfer data from Behance BAM! dataset during training. 


\xhdr{ResNet-50.} We implement \mdg\  on Resnet50 model (Table~\ref{tbl:pacs-r50}) used by the ERM in DomainBed. Adding \mdg\  loss regularization improves the accuracy of DomainBed, from 85.7 to 87.5 with \texttt{MDGHybrid}. Also, \hybrid\ performs better than the prior approaches using Resnet50 architecture, except RSC~\cite{huang2020self}, whose results (87.83) are close to ours (87.52). Note that we chose a subset of the best-performing baselines for Table~\ref{tbl:pacs-r18}, \ref{tbl:pacs-r50}; an extensive comparison with other works is in Suppl.~\ref{sm:resnet}. Suppl.~\ref{sm:alexnet} gives the results using AlexNet network, and a t-SNE plot (Figure~\ref{fig:t-sne}) to show the quality of representation learnt by MatchDG.



\begin{table}[tb]
\caption{Chest X-Rays data. As an upper bound, training ERM on the target domain itself yields 73.8\%, 66.5\%, and 59.9\% accuracy for RSNA, ChexPert, and NIH respectively.}
\centering
\footnotesize
\begin{tabular}{@{}l | c | c |c@{}}
\toprule
& RSNA & ChexPert & NIH   \\ \midrule

ERM      & 55.1 (2.93) & 60.9 (0.51) & 53.4 (1.36)                   \\ 
IRM     &  57.0 (0.75) & 63.3 (0.25) & 54.6 (0.88)                                      \\
CSD     & 58.6 (1.63)  & \textbf{64.4 (0.88)} & 54.7 (0.13)                                            \\
\texttt{RandMatch}  & 56.3 (3.38) & 55.3 (2.25) & 53.1 (0.13)                                               \\ 
\texttt{MatchDG}   & 58.2 (1.25) & 59.0 (0.25) & 53.2 (0.65)                                            \\
\texttt{MDGHybrid} & \textbf{64.3} (0.75) & 60.6 (0.25)  & \textbf{57.6} (0.13)                                             \\

\bottomrule
\end{tabular}
\label{tbl:chestxray}
\end{table}

\subsection{Chest X-rays dataset} Table~\ref{tbl:chestxray} provides results for the  Chest X-rays dataset, where the spurious correlation of vertical translation with the class label in source domains may lead the models to learn an unstable relationship. With RSNA as the target domain, ERM obtains 79.8\%, 81.8\% accuracy on the source domains while its accuracy drops to 55.1\% for the target domain. In contrast, \hybrid\ obtain the highest classification accuracy (8 \% above ERM), followed by CSD and \mdg; while methods like ERM and IRM are more susceptible to spurious correlation. However, on ChexPert as the target domain, CSD and IRM do better than ERM while matching-based methods are not effective. We conjecture these varying trends might be due to the inherent variability in images in the source domains, indicating the challenges of building domain generalization methods for real-world datasets.

\section{Conclusion}
We presented a causal view of domain generalization that provides an object-conditional objective. Simple matching-based methods perform competitively to state-of-the-art methods on PACS, indicating the importance of choosing the right invariance. The proposed MatchDG uses certain assumptions when objects are unknown. More work needs to be done to develop better matching methods, as indicated by the mixed results on the Chest-Xrays dataset. 


\xhdr{Acknowledgements.} 
We would like to thank Adith Swaminathan, Aditya Nori, Emre Kiciman, Praneeth Netrapalli, Tobias Schnabel, Vineeth Balasubramanian and the reviewers who provided us valuable feedback on this work. We also thank Vihari Piratla who helped us with reproducing the CSD method and other baselines.



\clearpage
\bibliography{example_paper}
\bibliographystyle{icml2021}

\clearpage
\appendix


\section{Synthetic Data (Slab Dataset and Simple Counter-example)}

\begin{table*}[h]
\footnotesize
\caption{Hyper parameter selection details for the slab dataset. We mention the Optimal Value for each hyper parameter and the Range used for grid search. We leave the optimal value for Epochs as blank since we do early stopping based on the validation loss, with the total number of epochs for each model as 100.}
\centering
\begin{tabular}{@{}l | l |c c @{}}
\toprule
Method                                                                            & \begin{tabular}[c]{@{}c@{}} Hyper Parameter \end{tabular} & \begin{tabular}[c]{@{}c@{}}Optimal Value\end{tabular}  &  \begin{tabular}[c]{@{}c@{}}Range\end{tabular} 
\\ 

\midrule

\multirow{3}{*}{\begin{tabular}[c]{@{}l@{}} DANN \end{tabular}}    

& \begin{tabular}[c]{@{}c@{}} Lambda \end{tabular}                                                      &  0.01 &  [0.01, 0.1, 1.0, 10.0, 100.0]   \\

& \begin{tabular}[c]{@{}c@{}} Gradient Penalty \end{tabular}                                                      &  0.1 &  [0.01, 0.1, 1.0, 10.0]   \\

& \begin{tabular}[c]{@{}c@{}} Discriminator Steps \end{tabular}                                                      &  4 &  [1, 2, 4, 8]   \\

\midrule

\multirow{3}{*}{\begin{tabular}[c]{@{}l@{}} CDANN \end{tabular}}    

& \begin{tabular}[c]{@{}c@{}} Lambda \end{tabular}                                                      &  0.01 &  [0.01, 0.1, 1.0, 10.0, 100.0]   \\

& \begin{tabular}[c]{@{}c@{}} Gradient Penalty \end{tabular}                                                      &  1.0 &  [0.01, 0.1, 1.0, 10.0]   \\

& \begin{tabular}[c]{@{}c@{}} Discriminator Steps \end{tabular}                                                      &  2 &  [1, 2, 4, 8]   \\

\midrule

\multirow{1}{*}{\begin{tabular}[c]{@{}l@{}} MMD \end{tabular}}    

& \begin{tabular}[c]{@{}c@{}} Lambda \end{tabular}                                                      &  0.1 &  [0.1, 1.0, 10.0]   \\

\midrule

\multirow{1}{*}{\begin{tabular}[c]{@{}l@{}} C-MMD \end{tabular}}    

& \begin{tabular}[c]{@{}c@{}} Lambda \end{tabular}                                                      &  0.1 &  [0.1, 1.0, 10.0]   \\

\midrule

\multirow{1}{*}{\begin{tabular}[c]{@{}l@{}} CORAL \end{tabular}}    

& \begin{tabular}[c]{@{}c@{}} Lambda \end{tabular}                                                      &  0.1 &  [0.1, 1.0, 10.0]   \\

\midrule

\multirow{1}{*}{\begin{tabular}[c]{@{}l@{}} C-CORAL \end{tabular}}    

& \begin{tabular}[c]{@{}c@{}} Lambda \end{tabular}                                                      &  0.1 &  [0.1, 1.0, 10.0]   \\

\midrule

\multirow{1}{*}{\begin{tabular}[c]{@{}l@{}} RandMatch \end{tabular}}    

& \begin{tabular}[c]{@{}c@{}} Lambda \end{tabular}                                                      &  1.0 &  [0.1, 1.0, 10.0]   \\

\midrule

\multirow{1}{*}{\begin{tabular}[c]{@{}l@{}} PerfMatch \end{tabular}}    

& \begin{tabular}[c]{@{}c@{}} Lambda \end{tabular}                                                      &  1.0 &  [0.1, 1.0, 10.0]   \\
\bottomrule 
\end{tabular}
\label{tbl:hyperparam-details-slab}
\end{table*}

\subsection{Implementation Details for the Slab Dataset}
\label{app:syn-implement}

\paragraph{Dataset}

The synthetic slab dataset (Section 3.2) consists of a binary label $y$ and 2-dimensional features; one feature has a linear relationship with $y$ while the other has a more complex ``slab'' relationship with $y$. The features vary in their simplicity, a measure of the simplicity of the feature is given by the number of linear pieces in the optimal classification/decision curve (Figure 1,~\cite{shah2020pitfalls}). Hence, the linear features are simpler as they only have 1 linear piece in the optimal decision boundary, as opposed to the slab features that have $k$ linear pieces in the piecewise linear optimal decision boundary. 

The synthetic slab dataset was introduced for detecting simplicity bias in neural networks~\cite{shah2020pitfalls}, to demonstrate that neural networks trained with SGD learn the simpler linear feature as opposed to the slab feature. We extend this dataset for the domain generalization (DG) task, by making the linear block features spurious due to domain-dependent noise addition as described below. The effect of the slab feature on $y$ remains the same across domains. Presence of the spurious linear feature should enable an ideal DG algorithm to differentiate it from the stable slab feature, and break the simplicity bias in neural networks. However, the invariance introduced by different DG methods can have a big impact. Using this dataset,  we show that the class-conditional distribution matching constraint is not sufficient (section ~\ref{sec:slab-dataset}); it is possible to satisfy the constraint using the  spurious linear feature too.

The linear block features contain the positive (y=1) and the negative (y=0) labels sampled from uniform distributions $U(0.1, 1.0)$ and $U(-1.0, -0.1)$ respectively. To make the linear features spurious, we add noise to the linear features s.t. data points are sampled from $U(-0.1, 0.1)$ with probability p, and sampled from $U(0.1, 1.0)$ (y=1) or $U(-1.0, -0.1)$ (y=0) with probability $1-p$. The $k$-slab feature ranges from [-1, 1] and within it has different ``slabs'' corresponding to uniform distributions of the feature's value, conditioned on class label $y$. The labels for these slabs alternate between  positive (y=1) and negative (y=0)  as the numeric value of the slab feature increases. The length for each slab is given by $\frac{2 - m*(k-1)}{k}$ where $k$ is the total number of slabs, and $m$ is the margin between two slabs. 

\textit{Linear Block Feature:}

\[
p_{l}(x|y=0)= 
\begin{cases}
    U(-0.1, 0.1) & \text{with prob.} \quad p\\
    U(-1.0, -0.1) & \text{with prob.} \quad 1-p\\
\end{cases}
\]

\[
p_{l}(x|y=1)= 
\begin{cases}
    U(-0.1, 0.1) & \text{with prob.}\quad p\\
    U(0.1, 1.0) & \text{with prob.}\quad 1-p\\
\end{cases}
\]

\textit{Slab Block Feature:}

Let $k$ be the total number of slabs, and $m$ be the margin between two slabs.

Slab length: $L= \frac{2 - m*(k-1)}{k}$ 

Start index: $I(i)= -1 + i*L$
\[
p_{s}(x|y=0)= 
\begin{cases}
    U( I(i), I(i) + L) & \text{for i} \in \{0, 2, 4, ...\}\\
\end{cases}
\]

\[
p_{s}(x|y=1)= 
\begin{cases}
    U( I(i), I(i) + L) & \text{for i} \in \{1, 3, 5, ...\}\\
\end{cases}
\]

We also a constant (domain independent) noise to the relationship between the slab feature and the class label by flipping the original label with probability $p_{s}$. 

\paragraph{Source and Target Domains} We generate two source data domains with noise probabilities $p$ as $0.0$ and $0.1$, and generate the target domain with complete noise $p=1.0$, rendering the linear block feature not informative of the label in the target domain. However, the slab block features have a stable relationship with the labels across the multiple source and target domains. We choose $k=7$, $m=0.1$, and $p_{s}=0.1$ for the slab block features in our experiments. We sample $1k$ data points per domain, which leads to $2k$ training data points (source domain with $p$ as $0.0$ and $0.1$), and $1k$ test data points (target domain with $p$ as $1.0$). Also, for hyperparameter tuning (model selection), we sample additional 250 data points per source domain as the validation set.

\paragraph{Model Architecture}
The overall architecture consists of a representation network along with a classification network, detailed below. Input Dim refers to the input data dimension, which is 2 dimensional (linear block feature, slab block feature). Num Classes refers to the total number of output classes, which is binary classification for the synthetic slab dataset. We refer a fully connected dense layer by FC layer, with the input and output dimensions for that layer in brackets. 

\textit{Representation Network}
\begin{itemize}
    \item FC layer: (Input Dim, 100)
    \item ReLU activation
\end{itemize}

\textit{Classification Network}
\begin{itemize}
    \item FC layer: (100, 100)
    \item FC layer: (100, Num Classes)
\end{itemize}

For methods like DANN~\cite{ganin2016dann}, and CDAAN~\cite{li2018adversarialcada}, which also require domain discriminators, we use the same architecture for them as that of the classification network.

\paragraph{Methods}

We use Cross-Entropy for the classification loss in ERM and all the other methods. The regularization penalty of all the methods is placed on the output of the representation network.

For the methods DANN~\cite{ganin2016dann}, CDANN~\cite{li2018adversarialcada}, MMD~\cite{li2018adversarialfeature}, CORAL~\cite{sun2016deep}, we used their implementation available in DomainBed~\cite{gulrajani2020domainbed}. We extended the implementation of MMD and CORAL from DomainBed to their class conditional versions, C-MMD, and C-CORAL. The extension to class conditional version was done by computing their respective penalty over domains conditioned on a particular class label.

For RandMatch~\cite{motiian2017ccsa} and PerfMatch, we use $l_{2}$ distance for $\operatorname{dist}$ in (Eq: ~\ref{eq:erm-pmatch-loss}). The match function $\Omega$ in the RandMatch algorithm is defined as randomly matching any two data points across domains with the same class label. For the PerfMatch algorithm, the match function $\Omega$ accepts two data points across domains with the same slab id as valid matches (the slab id corresponds to the value of the causal feature: two inputs with the same slab id have similar causal features).

\paragraph{Note on method selection}
Our objective with the synthetic slab dataset is to compare the performance of conditional distribution matching (CDM) methods to that of Perfect Match. We chose the above mentioned CDM methods for experimentation since the other CDM methods~\cite{,li2018conddomaingen, ghifary2016scatter, hu2019discriminant} mentioned in the related works (Section 2, main paper) did not have their implementation publicly available. We found the implementation of CDANN~\cite{li2018adversarialcada} in the DomainBed~\cite{gulrajani2020domainbed} repository, which also provided implementation for the unconditional distribution matching methods like MMD~\cite{li2018adversarialfeature}, and CORAL~\cite{sun2016deep}. Hence, we extended MMD and CORAL to their class-conditional variant using their original implementation from the DomainBed repository.

\paragraph{Hyperparameter Tuning}

All the methods were trained using SGD, with batch size 128, learning rate 0.1 and weight decay 5e-4. We train each method for 100 epochs and do early stopping based on the validation loss.

Further details regarding the tuning of hyperparameters specific to each method's regularization technique are provided in Table~\ref{tbl:hyperparam-details-slab}. The loss objective of all the methods can be written as ERM + $\lambda$*Regularization Penalty; and we provide the optimal values and grid range for the hyperparam $\lambda$ in Table~\ref{tbl:hyperparam-details-slab}. Also, some methods like DANN, C-DANN have additional hyperparams, which are specified in the same table. The grid search range for methods that were implemented using DomainBed~\cite{gulrajani2020domainbed} is taken from the Table 8 in the DomainBed paper.

\subsection{Simple counter-example and its relationship to the MatchDG assumption.}
The \mdg\ method depends on Assumption 1 (Section 5) which requires that same-class inputs across domains are closer in causal features than different-class inputs. Note that the example in Section~\ref{sec:simple-syn} does not satisfy this assumption. However, there exist many variations of the setup that do follow the MatchDG assumption, and still class-conditional methods cannot recover the true causal feature. For instance, by setting $|x'_c| = |x_c|+\kappa$ where $\kappa > 1.5$ and $\alpha_1=\kappa+1, \alpha_2=\kappa+2$ for domain 1 and domain 2 respectively, the train domains satisfy the MatchDG assumption.

Overall, the goal of the simple example in Section~\ref{sec:simple-syn} is to show that there exist datasets where class-conditional methods would not work, but Perfect-Match does. MatchDG's assumption  works in a subset of these datasets.  In future work,  matching-based methods can be developed that relax the MatchDG assumption.

\section{Theory and Proofs}
\subsection{Constructing the causal graph}
\label{app:graphdetails}
When considering classification tasks, there are two viewpoints on whether the features cause the class label, or whether the class labels cause the features. \cite{gong2016domainadapt,magliacane2018domainadaptcausal,rojas2018invarianttheory} assume a generative process where the true class label determines the features in the observed data. In contrast, \cite{peters2016causal,arjovsky2019irm} consider a generative process where the features are used to assign a label, e.g., when manually labelling a set of images. We believe that both mechanisms are possible, depending on the context. In particular, it is plausible that the true class label $Y_{true}$ causes the features, but it is not observed. Instead, what is observed is the output of a manual labelling process, where the  features are used to  label each input with its class $Y$~\cite{arjovsky2019irm}.  

Given these differences, we  construct a causal graph (Figure \ref{fig:scm-all}) that includes both $Y_{true}$ and $Y$ (as in \cite{heinze2017conditional}), and is consistent with both viewpoints about the direction of the causal mechanism. Importantly, all d-separation results reported in the main text hold true irrespective of whether we choose $Y$ or $Y_{true}$  as the class label. We use $Y$ as the label in the main text, since it corresponds to many settings where the observed class label is a result of a (possibly noisy) manual labelling process. 

In addition, we chose to represent $X_C$ and $X_A$ as near-to-final features, that are combined using a simple operation to generate the observed features $X$. Under this representation, the object $O$ does cause $X_A$; $X_A$ is produced by combination of the domain and the object. Another equally valid construction is to assume that $X_A$ contains only the domain information, and a more complex operation generates the observed features using $X_C$ (object information) and $X_A$. The corresponding causal graph will omit the edge from object $O$ to $X_A$. 
Both these graphs are allowed by our framework. All d-separation results reported in the main text hold true irrespective of whether there exists an edge from $O$ to $X_A$.

\subsection{D-separation}
\label{app:d-sep}
We first expand on the d-separation definition, providing a few examples that illustrate conditional independence implications of specific graph structures in \figref{fig:dsep-examples}. We use these three conditions for the proofs below. 


\begin{figure}[h]
    \centering
    
    \subfigure[][Chain: $A\not \indep B; A \indep B| C$]{
        \includegraphics[scale=0.25]{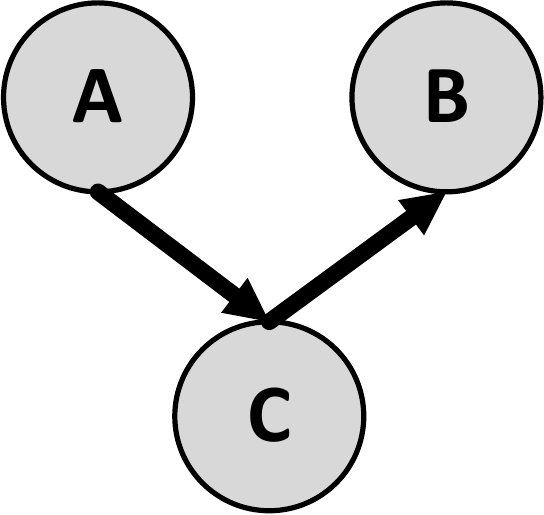}
        }
        
    \subfigure[][Fork:  $A\not \indep B; A \indep B| C$]{
        \includegraphics[scale=0.25]{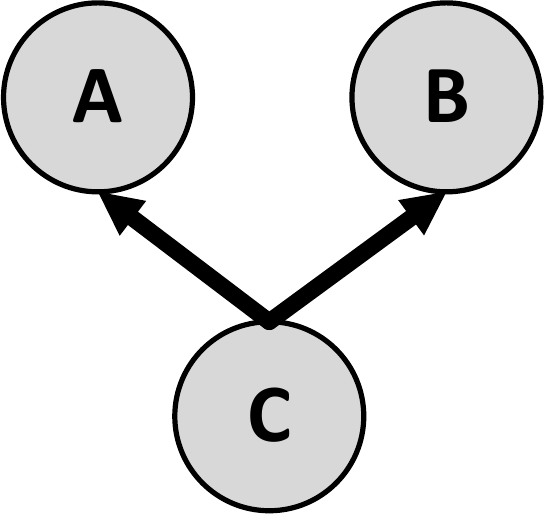}
        }
        
    \subfigure[][Collider: $A \indep B; A \not \indep B |C$]{
        \includegraphics[scale=0.25]{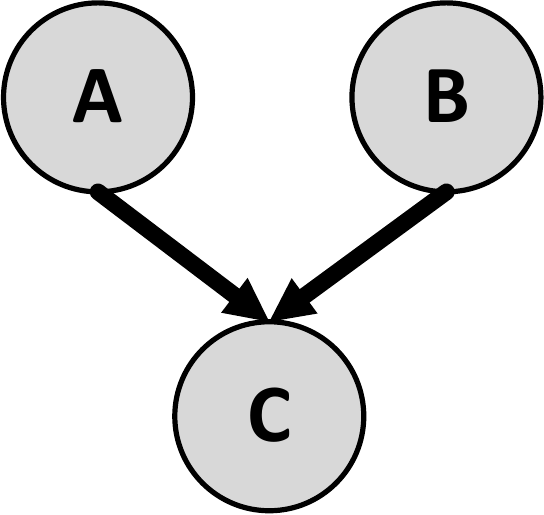}
        }
        
    \caption{Causal graphs with the node $C$ as a chain, fork, or a collider. By the d-separation criteria, $A$ and $B$ are conditionally independent given $C$ in (a) and (b). In (c) however, $A$ and $B$  are independent but become conditionally dependent given $C$.} 
    \label{fig:dsep-examples}
\end{figure}

\begin{definition}
\label{def:dsep}
    \textbf{d-separation} \cite{pearl2009book}: Let A,B,C be the three non-intersecting subsets of nodes in a causal graph $\mathcal{G}$. For any path between two nodes, a collider is a node where arrows of the path meet head-to-head. A path from $A$ to $B$ is said to be \emph{blocked} by $C$ if either a non-collider on the path is in $C$,  or there is a collider on the path and neither the collider nor its descendants are in $C$. 
    
    If all paths from $A$ to $B$ are blocked, then A is d-separated from B by C: $\operatorname{dsep}(A, B, C)  \Rightarrow A \indep B | C$. 
\end{definition}

\subsection{Proof of Proposition 1}
\label{app:proofprop1}

Proposition 1 relates to the domain generalization setup, as described in Section~\ref{sec:simple-syn}, where the causal feature determines  $y$ label without any noise. The distribution of the non-causal feature varies across domains. The proof uses the entropy formulation of distribution-matching methods, as done by \citet{akuzawa2019adversarial}. 
\classcondprop*

\begin{proof}
We can write class-conditional invariant models as optimizing two objectives: minimize the error on the training data (ERM objective), and learn a representation $\Phi(\vecx)$ that is independent of domain given the class label (class-conditional invariant).

Let us focus on the second objective, which  be interpreted as 
2) 
maximizing the entropy of domain given class label and representation $H(d|y,\Phi(\vecx))$. 
Let $\Phi_{2}$ be the optimal representation for the class-conditional invariant. We can write, 
\begin{equation}
    \Phi_{2} = \arg \max_\Phi H(d|y, \Phi(\vecx))
\end{equation}

Since $H(d|y, \Phi(\vecx)) \leq H(d|y)$ using the property of entropy, the optimal $\Phi_2$ satisfies, 
\begin{equation}
\label{eq:cdm}
    H(d|y, \Phi_2(\vecx)) = H(d|y)
\end{equation}

Now two cases arise; $x_c \indep D| Y $ or $x_c \not \indep D| Y$. We assume the former.  
If $X_C$ is independent of domain conditioned on the class label, then
\begin{equation}
\begin{split}
    H(d|y)&=H(d|y, x_c)
\end{split}
\end{equation}
Here domain $d$ is independent of both $x_c$ and $\Phi_2(\vecx)$, conditional on $y$. Since causal features $x_c$ cannot be caused by the representation $\Phi_2(\vecx)$, it cannot be a collider (Definition~\ref{def:dsep}) in any graph connecting $d$, $x_c$ and $\Phi_2(\vecx)$. Therefore, conditioning on it does not remove the independence between $d$  and $\Phi_2(\vecx)|y$ (conditioned on $y$). Hence, we condition on Eq~\ref{eq:cdm} with $x_c$ and obtain, 
\begin{equation}
\label{eq:cdm|x_c}
    H(d|y, x_c)=H(d|y, x_c, \Phi_2(\vecx)) 
\end{equation}

 Plugging it into the above equations, we obtain 
\begin{equation}
\begin{split}
    H(d|y) &=    H(d|y, x_c)
  =  H(d| y, x_c, \Phi_2(\vecx))
  \end{split}
\end{equation} 

Also since there is no label noise,  $x_c$ can achieve zero error for predicting the label $y$. That is, $x_c$ contains all information about $y$, and thus we can remove $y$ from the above equation, 
\begin{equation}
    H(d|x_c) = H(d|\Phi_2(\vecx), x_c)
\end{equation}

This implies that the learnt representation $\Phi_2(\vecx)$ is independent of the domain given $x_c$; thus $\Phi_2(\vecx)$ depends on $x_c$ and not on any other feature that changes with domain. 
\end{proof}

\subsubsection{Remarks based on Proposition 1}
\textbf{If $x_c$ is not independent of domain given class label. }However, if $X_C$ is not independent of domain given the class label (i.e.,  $P(x_c|y)$ changes across domains), then $H(d|y) > H(d|y, x_c)$. Using the equality from Eq~\ref{eq:cdm|x_c}, we obtain, 
\begin{equation}
\begin{split}
    H(d|y) = H(d|y, \Phi_2(\vecx)) &>   H(d|y, x_c) \\
    H(d|y, \Phi_2(\vecx)) &\geq H(d|\Phi_2(\vecx), y, x_c)
\end{split}
\end{equation} 
 After removing $y$ as in Eq. 10, $H(d| x_c)$ and $H(d|\Phi_2, x_c)$ may not be equal. In particular, the ground-truth representation $\Phi_{GT}(\vecx)=x_c$ does not satisfy the class-conditional invariant: $H(d|y, \Phi_{GT}(\vecx)) =H(d|y, x_c) \neq H(d|y)$. 
 
 Hence, to learn $x_c$ as the representation,  we need a separate constraint, $H(d| x_c)=H(d|\Phi, x_c)$; i.e. domain and representation should be independent conditioned on $x_c$. 

\xhdr{Implications for the slab dataset (Section~\ref{sec:slab-dataset}).} In the slab dataset, $x_c=x_2$ and $x_2$ is independent of the domain given the class label ($X_2 \indep D|Y$). We also see that $\Phi(\vecx)=x_2$ satisfies the class-conditional invariant: $H(d|y, x_2)=H(d|y)$ since $X_2 \indep D|Y$. By Proposition 1, the class-conditional invariant should lead to a representation that satisfies $H(d|x_c)=H(d|\Phi(\vecx), x_c)$. However, the same constraint can also be achieved by setting $\Phi(\vecx)=x_1$ by shifting the distribution of $x_1$ slightly. And since there is a simple, linear correlation between $x_1$ and the class label $y$, empirically class-conditional methods end up learning a representation dependent on $x_1$.



\subsection{Proof of Proposition 2}
\label{app:proofprop2}

\identifyprop*

\begin{proof}
To prove non-identifiability, it is sufficient to show a counter-example where the same structural equations (and hence same observed and interventional distributions over $Y, X, D, O$)  correspond to two different values of $X_C$.

From Section~\ref{sec:dgp}, the SCM leads to the following structural equations,
\begin{align*}
     o &:= g_o (y_{true}, \eps_o, \eps_{od}) &  & \vecx_c: = g_{xc}(o)  \\
     \vecx_a &:= g_{xa}(d,o, \eps_{xa}) & \vecx := g_x(\vecx_c, \vecx_a, \eps_x) & y := h(\vecx_c, \eps_y) 
\end{align*}

Substituting for $\vecx_c$ in the SCM equations, we obtain, 
\begin{equation}
    \begin{split}
    y &=h(g_{xc}(o), \eps_y)\\ \vecx &=g_x(g_{xc}(o), g_{xa}(d,o,\eps_{xa}), \eps_x)\\
    \end{split}
\end{equation} 

Given a value of object variable $o$, note that $g_{xc}$ determines $x_c$. We now proceed to show that different values of $g_{xc}$ are possible given the same structural equations between the observed variables $Y, X, D, O$.  Specifically, by choosing $g_x$ and $h$ appropriately, different values of $g_{xc}$  can lead to the same observed values for $(y,d,o,x)$.  

\noindent \textbf{A simple counter-example.} Suppose the following SCM equations, 
\begin{equation}
\begin{split}
    y &= h(g_{xc}(o)) \\
    x &=g_1(g_{xc}(o))+g_2(o,d)
\end{split}
\end{equation}

Introducing $h^*=h\circ g_{xc}$ and $g_1^*=g_1 \circ g_{xc}$, we can rewrite the above equations as, 
\begin{equation}
\begin{split}
    y &=h^*(o) \\
    x &=g^*_1(o)+g_2(o,d)
\end{split}
\end{equation}
then any $g_{xc}$ is applicable as long as we set $h$ such that $h(g_{xc})(o)=h^*(o)$ and set $g_1$ such that $g_1(g_{xc}(o))=g^*_1(o)$. In particular,  if the SCM equations are $y=o$, $x=o+o*d$, and we  define $h=g_1=g_{xc}^{-1}$, then $g_{xc}$ can be any invertible function. 
Hence, different values of $x_c=g_{xc}(o)$ will lead to the same structural equations over $Y, X, D, O$, and therefore the same observed and interventional distributions. 
\end{proof}

\subsection{Proof of Theorem 1}
\label{app:thm1proof}
\perfectmatchtheorem*

\begin{proof}
    \textbf{CLAIM 1. }
    The matching condition can be written as:
    \begin{equation} \label{eq:pmatch-loss-c}
    \small
        C(\vecphi)= \min_{\vecphi} \sum_{d,d'\in
        D_m}\lim_{\substack{n_d\rightarrow \infty\\ n_{d'} \to \infty}}\sum_{\Omega(j,k)=1; d\neq d'} \operatorname{dist}(\vecphi(\vecx_j^{(d)}), \vecphi(\vecx_k^{(d')}))
    \end{equation}
    where $\Omega(j,k)=1$ for pairs of inputs $\vecx_j$ and $\vecx_k$ from two different domains $d$ and $d'$ that correspond to the same object. The distance metric $\operatorname{dist}$ is non-negative, so the optimal $\vecphi$ is when $C(\vecphi)$ is zero. 
As in the SCM from \figref{fig:scm-withnoise},  let $X_c$ represent a feature vector such that it is generated based only on the object $O$ and that it leads to the optimal classifier in ~\eqref{eq:rep-ideal-loss}.  From Sections~\ref{sec:dgp} and \ref{sec:objcond}, we know that $X_c \indep D|O$ and that $x_c=g_{xc}(o)$. Thus,  $x_c$ is the same for inputs from the same object and we can write:
\begin{equation}
        \operatorname{dist}(\vecx_{c,j}^{(d)}, \vecx_{c,k}^{(d')})=0 \text{ \  \ }\forall d,d' \in {D}_m \text{ such that } \Omega(j,k)=1
\end{equation}

Hence,  $\vecphi(\vecx) = \vecx_c$ leads to zero regularizer term and is one of the optimal minimizers for $C(\vecphi)$. 

\textbf{CLAIM 2. }    Further, we show that any other optimal $\vecphi$ is either a function of $\vecxc$ or a constant for all inputs.  We prove by contradiction. 

Let $X_A$ represent the set of unobserved high-level features that are generated based on both the object $O$ and the domain $D$. From the  SCM from \figref{fig:scm-withnoise}, a feature vector $X_a \subseteq X_A$ is independent of $X_c$ given the object, $X_a \indep X_c| O$, and $x_a=g_{xa}(d,o, \eps_{xa})$. Further, let there be an optimal $\vecphi_{a}(\vecx)$ for $C(\vecphi)$ such that it depends on some $X_a \subseteq X_A$ (and is not trivially a constant function). 
Since $\vecphi_{a}$ is optimal, $\vecphi_{a}(\vecx_j^{(d)}) = \vecphi_a(\vecx_k^{(d')})$ for all $d,d'$ such that $\Omega(j,k)=1$, where inputs $\vecx_j$ and $\vecx_k$ correspond to the same object. 

Let us assume that there exists at least one object $o$ for which the effect of domain is stochastic. That is, due to domain-dependent variation, $P(X_a=x_a|D=d, O=o)<1$.  for some $d$ and $o$. Now consider a  pair of inputs $\vecx_l^{(d)}$  and $\vecx_i^{(d')}$ from the same object $o$ such that $\Omega(l,i)=1$, and their corresponding
representations are $\vecphi_a(\vecx_l^{(d)})$ and $\vecphi_a(\vecx_i^{(d')})$. Due to domain-dependent variation,  with non-zero probability, the high-level $X_a$ features are not the same for these two input data points, $x_{a,l}^{(d)} \neq x_{a,i}^{(d')}$. Since $\vecphi$ is a deterministic function of $\vecx$ that is not independent of  $X_a$, if an input $\vecx$ has a different $X_a$, its value of $\vecphi(\vecx)$ will also be different.  Thus, with non-zero probability, we obtain that $\vecphi(\vecx_l^{(d)}) \neq \vecphi(\vecx_i^{(d')})$, unless the effect of $X_a$ is a constant function. Hence, a contradiction and optimal $\vecphi$ cannot depend on any $X_a \subseteq X_A$ that are generated based on the domain.
    
Therefore, an optimal solution to $C(\vecphi)$ can only depend on $X_c$. However, any function of $X_c$ is optimal, including trivial functions like the constant function (that will have low accuracy). Below we show that using the ERM term in \eqref{eq:erm-pmatch-loss} ensures that the optimal solution contains only those functions of $X_C$ that also maximize accuracy.

Using \eqref{eq:erm-pmatch-hardloss}, the empirical optimizer function can be written as (where we scale the loss by a constant $n=\sum_d n_d$, the total number of training data points):
\begin{align} 
\hat{f}_{pmatch} &=  \arg \min_{h, \vecphi}\frac{1}{n} \sum_{d=1}^{m} \lim_{n_d \rightarrow \infty} L_d(h(\vecphi(X)), Y) \\& \texttt{ s.t.} \sum_{\Omega(j,k)=1; d\neq d'} \operatorname{dist}(\vecphi(\vecx_j^{(d)}), \vecphi(\vecx_k^{(d')})) = 0   \label{eq:thm1-proof-fpmatch} \\
                 &=\arg \min_{h,\psi} \frac{1}{n} \sum_{d=1}^{m} \lim_{n_d \rightarrow \infty}L_d(h(\psi(X_c)), Y) \nonumber \\
                 &=\arg \min_{f} \frac{1}{n} \sum_{d=1}^{m} \lim_{n_d \rightarrow \infty}L_d(f(X_c), Y) \label{eq:thm1-proof-simplefpmatch}
\end{align}
where $\psi(X_c)$ denotes all functions of $X_c$ that are optimal for \eqref{eq:pmatch-loss-c}, and the last equality is because $h \circ \psi$ can be written as $f=h\circ \psi$.
Since we assume that $L$ is a P-admissible loss function, its minimizer is the conditional expected value. Thus,  for any domain $d$,  $\arg \min_{f} \lim_{n_d \rightarrow \infty} \frac{1}{n_d} L_d(f(X_c), Y)= \mathbb{E}[Y|X_c, D]$. 
Further, by d-separation, $Y\indep D|X_c$.  Therefore, $\mathbb{E}[Y|X_c, D]=\mathbb{E}[Y|X_c] $. 
The above equation indicates that the loss minimizer function on any domain is independent of the domain. Thus, for the $m$ training domains, we can write:
\begin{equation}\label{eq:thm1-proof-condexp}
\begin{split}
    \arg \min_{f \in \mathcal{F}} \lim_{n_d \rightarrow \infty} \frac{1}{n_d}L_d(f(X_c), Y) 
    &= \arg \min_{f \in \mathcal{F}} \mathbb{E}[l(f(\vecx_c), y)] \\
   &=\mathbb{E}[Y|X_c] 
   \text{ \ \ } \forall d \in D_m
\end{split}
\end{equation}
Now \eqref{eq:thm1-proof-simplefpmatch} can be rewritten as,
\begin{equation}
\begin{split}
\hat{f}_{pmatch}  &=\arg \min_{f} \frac{1}{n} \sum_{d=1}^{m} \lim_{n_d \rightarrow \infty} \frac{L_d(f(X_c), Y)}{n_d} n_d \\
                    &=\arg \min_{f} \sum_{d=1}^{m} \lim_{n_d \rightarrow \infty} \frac{L_d(f(X_c), Y)}{n_d} \frac{n_d}{n}
\end{split}
\end{equation}

From the equation above, the loss for $\hat{f}_{pmatch}$ can be considered as a weighted sum of the average loss on each training domain where the weights are all positive. Since $E[Y|X_c]$ minimizes the average loss on each domain as $n_d \to \infty$, it will also minimize the overall weighted loss for all values of the weights.  Therefore, for any dataset over $m$ domains in $D_m$, $\mathbb{E}[Y|X_c]$ is the optimal function that minimizes the overall loss. 

Moreover, we can also write $f^*$ as:

\begin{equation}\label{eq:rep-ideal-loss-supp}
\begin{split}
    f^* &= \arg \min_{f \in \mathcal{F}} \mathbb{E}_{(d, \vecx,y)}[l(y, f(\vecx))] \\
        &= \arg \min_{h \in \mathcal{F}} \mathbb{E}_{(d, \vecx,y)}[l(y, h(\vecxc))]\\
        &= \arg \min_{h \in \mathcal{F}} \mathbb{E}_{(\vecx,y)}[l(y, h(\vecxc)] 
        = \mathbb{E}[Y|X_c]
\end{split}
    \end{equation}

    where we utilize \eqref{eq:thm1-proof-condexp} and that the loss function is P-admissible. Hence, $f^*=\mathbb{E}[Y|X_c]$ is the loss-miniziming function for the loss in \eqref{eq:thm1-proof-simplefpmatch}. 

    Finally, using a Lagrangian multiplier, minimizing the following soft constraint loss is equivalent  to minimizing \eqref{eq:thm1-proof-fpmatch}, 
    for some value of $\lambda$.
\begin{multline} 
    \hat{f}_{\tt pmatch} = \lim_{\forall d\in D_m n_d\to \infty} \arg \min_{h, \vecphi} \sum_{d=1}^{m} L_d(h(\vecphi(X)), Y) \\ + \lambda \sum_{\Omega(j,k)=1; d\neq d'} \operatorname{dist}(\vecphi(\vecx_j^{(d)}), \vecphi(\vecx_k^{(d')})) 
\end{multline}
The result follows. 
\end{proof}

\textbf{Comment on Theorem 1.} In the case where the effect of a domain is also deterministic, it is possible that  $P(X_a|O,D)=1$ (e.g., in artificially created  domains like Rotated-MNIST where every object is rotated by the \emph{exact} same amount in each domain). In that case Theorem 1 does not apply and it is possible to learn a  representation $\vecphi_a$ that depends on $X_a \subseteq X_A$ and still minimizes $C(\vecphi)$ to attain $C(\vecphi)=0$. 
For example, with two training domains on Rotated-MNIST dataset ($0^{\circ}$, $\alpha^{\circ}$), it is possible to learn a representation that simply memorizes to ``un-rotate '' the $\alpha$ angle back to $0^{\circ}$. Such a representation will fail to generalize to domains with different rotation angles, but nonetheless minimizes $C(\vecphi)$ by attaining the exact same representation for each object. 

In practice, we conjecture that such undesirable $\vecphi_a$ are avoided by model-size regularization during training. As the number of domains increase, it may be simpler to learn a single transformation (representation) based on $X_c$ (and independent of $X_c$ features like angle) than  learn separate angle-wise transformations for each train domain.  

\subsection{Proof of Proposition 3}
\label{app:proofcorr1}

\xhdr{Domain-invariant representations.} ($\vecphi(\vecx) \indep D$)~\cite{muandet2013domain,li2018adversarialfeature, ganin2016dann}. 
Using d-separation on the SCM from \figref{fig:scm-withnoise}, $X_C \indep D$ is not sufficient since $O$ blocks the path between $X_C$ and $D$. While  ~\cite{zhao2019learninginvariant} argue that this condition fails when $Y$ is correlated with $D$, our analysis shows that  domain-invariant methods require a stronger condition that both class label and actual objects sampled  be independent of domain. 

\xhdr{Class-conditional domain-invariant.} ($\vecphi(\vecx) \indep D|Y$.)~\cite{li2018conddomaingen, ghifary2016scatter, li2018adversarialcada} Even in the ideal case where we observe  $Y_{true}$, d-separation on the SCM reveals  that  $X_C \not \indep D|Y_{true}$ due to a path through $O$. Thus, having the same distribution per class is not consistent with properties of $X_C$. 

Below we prove these results formally. 
\othermethodscor*
\begin{proof}
As in the SCM from \figref{fig:scm-withnoise},  let $X_c$ represent an unobserved high-level feature vector such that it is generated based only on the object $O$ and that it leads to the optimal classifier in ~\eqref{eq:rep-ideal-loss}.  From Sections~\ref{sec:dgp} and \ref{sec:objcond}, we know that $X_c \indep D|O$ and that $x_c=g_{xc}(o)$. 
    Following a similar proof to Theorem~\ref{thm:main-match} (Claim 1), we check whether $\vecphi(\vecx)=\vecx_c$ satisfies the invariance conditions required by the two methods. 
    \begin{enumerate}
        \item \textbf{Domain-invariant}: The required condition for a representation is that $\vecphi_{\tt DI}(\vecx) \indep D$. But using the d-separation criteria on the SCM in \figref{fig:scm-withnoise}, we find that $X_c \not \indep D$ due to a path through Object $O$.
        \item \textbf{Class-conditional domain-invariant}: The required condition for a representation is that $\vecphi_{\tt CDI} \indep D | Y$. However using the d-separation criteria on the SCM, we find that $X_c \not \indep D | Y$ due to a path through Object $O$ that is not blocked by $Y$ (nor by $Y_{true}$ if it is observed).
    \end{enumerate}
Therefore, under the conditions proposed by these methods, $X_c$ or any function of $X_c$ is not an optimal solution without making any additional assumptions. Hence, even with infinite samples, a method optimizing for these conditions will not retrieve $X_c$. 
\end{proof}

\section{Evaluation and implementation details}
In this section we describe implementation details for our proposed methods. We also discuss the evaluation protocol, including details about hyperparameters and cross-validation. 

\subsection{Implementation details}
\label{app:impl}

        For the implementation of \epmatch\ in  Eq. (\ref{eq:erm-pmatch-loss}); we use the cross-entropy loss for $L_{d}$ and $l_{2}$ distance for $\operatorname{dist}$ in Eq. (\ref{eq:erm-pmatch-loss}). Similarly, we implement the \ermatch\ with a match function $\Omega$ in Eq. (\ref{eq:erm-pmatch-loss}) that randomly matches data points across domains with the same class, For both methods, we consider the representation $\vecphi(\vecx)$ to be the last layer of the network. That is, we take $h$ to be identity function in Eq. (\ref{eq:erm-pmatch-loss}) for simplicity.  It is also possible to use the second-last or any other previous layer as a representation, but the last layer performed well in our experiments. Also, given a fixed data point, the match function $\Omega$ could select multiple data points as potential matches for it. In this case we use Eq. (\ref{eq:erm-pmatch-loss}) with stochastic matching, where we randomly select one match out of the potential     multiple matches.

We use SGD to optimize  the loss for all the datasets, with details about learning rate, epochs, batch size, weight decay etc. provided in the section~\ref{sm:hyperparam-tuning} ahead. For all the different methods, we sample batches from the data matrix consisting of data points matched across domains; hence we ensure an equal number of data points from each source domain in a batch. When training with \mdg, the underlying architecture for Phase 2 is kept  the same for ERM, \texttt{RandMatch}, \texttt{PerfMatch} for the respective task; with the details mentioned below for each dataset. The details for the Phase-1 architecture are specified in section~\ref{sm:hyperparam-tuning}, Table~\ref{tbl:mdg-p1}. \\

\xhdr{Rotated MNIST \& Fashion-MNIST.} The datasets contain rotations of  grayscale MNIST handwritten digits and fashion article images from $0^{\circ}$ to $90^{\circ}$ with an interval of $15^{\circ}$~\cite{ghifary2015multitaskae}, where each rotation angle represents a domain and the  task is to predict the class label. For Table~\ref{tbl:mnist}, we follow the setup in CSD~\cite{piratla2020efficient}, we report accuracy on $0^{\circ}$ and $90^{\circ}$ together as the test domain and the rest as the train domains. We use $2,000$ and $10,000$ training samples from each domain for rotated MNIST and Fashion-MNIST, and train models using Resnet-18 architecture (without pre training). We choose this as our primary setup and select $0^{\circ}$ and $90^{\circ}$ as our target domain, since these are known to be the most difficult domains to generalize~\cite{piratla2020efficient, motiian2017ccsa}. 

Further, we also evaluate on other setups of Rotated MNIST in prior works~\cite{motiian2017ccsa, gulrajani2020domainbed}, which involve six domains ($0^{\circ}$, $15^{\circ}$, $30^{\circ}$, $45^{\circ}$, $60^{\circ}$, $75^{\circ}$), and evaluate for each domain being the target domains with the remaining five used as source domains. We sample 1000 data points for each domain and evaluate using the LeNet architecture (Table \ref{tbl:mnist-lenet}) as per the setup proposed by \cite{motiian2017ccsa}.  Similarly, we sample all the 70,000 images in MNIST and evaluate using the custom architecture (Table \ref{tbl:mnist-domain-bed}) as per the setup proposed by \cite{gulrajani2020domainbed}. 

Another important distinction between different setups above is the use of different digits for the source and the target domains (\cite{piratla2020efficient}, \cite{gulrajani2020domainbed}), as opposed to the use of same digits across the source and the target domains in setup of \cite{motiian2017ccsa} which makes the task easier as it leaks information about the target domains.

Finally, for all the different setups proposed above, we create an additional validation set for each domain with 20\% percent size as of the training set for that domain. We use the validation set from the source domains for hyper parameter tuning. \\

 \xhdr{PACS.} This dataset contains total  $9991$ images from four domains: Photos (P), Art painting (A), Cartoon (C) and Sketch (S). The task is to classify objects over $7$ classes. Following~\cite{ dou2019masf}, we train $4$ models with each domain as the target using Resnet-18 (Table \ref{tbl:pacs-r18}), Resnet-50 (Table \ref{tbl:pacs-r50}) and Alexnet (Table \ref{tbl:pacs-alexnet}), with each architecture pre-trained on ImageNet. We also the following data augmentations~\cite{gulrajani2020domainbed} while training: Random Crop, Horizontal Flip, Color Jitter, and Random Gray Scale.
 \\

\xhdr{Chest X-ray.} We use Chest X-rays images from three different sources: NIH~\cite{wang2017chestx}, ChexPert~\cite{irvin2019chexpert} and RSNA~\cite{rsna}. The task is to detect whether the image corresponds to a patient  with Pneumonia (1) or not (0). For ease of interpretation, we balance the data such that there are equal number of images per class in each domain. Since majority of the images in each domain correspond to the class (0), we sample a subset of the images to ensure that there is no class imbalance in each domain. The dataset size for the different splits on each domain are described below:

\begin{itemize}
    \item NIH: Train (800), Validation (200), Test (400)
    \item ChexPert: Train (800), Validation (200), Test (400)
    \item RSNA: Train (800), Validation (200), Test (400)
\end{itemize}

Following prior works~\cite{cohen2020limits}, we use the pre-trained DenseNet-121 architecture for classification. We  use the following data augmentations: Random Crop and Random Horizontal Flip. We further create spurious correlations, all the images of the class 0 in the training domains are translated vertically downwards; while no such translation is done for the test domain. We translate the images in each source domain by a fixed amount, which varies over different source domains (NIH (45), ChexPert (35), RSNA (15)). This leads to a downward shift in the position of lungs in the images for the class 0 as compared to those for class 1, which could lead to models utilizing this spurious relative difference in position of lungs for the classifications task. 

\begin{table*}[th!]
\footnotesize
\caption{Hyper parameter selection details for all the datasets. We mention the Optimal Value for each hyper parameter and the Range used for grid search. We leave the optimal value for Epochs as blank since we do early stopping based on validation loss, with the total number of epochs for model training specified in the Range column. For the dataset PACS, since the optimal values differ for different test domains, we represent them separately in Table ~\ref{tbl:hyper-param-pacs} }
\centering
\resizebox{\textwidth}{!}{
 \begin{tabular}{@{}l | l |c c @{}}
\toprule
Dataset                                                                            & \begin{tabular}[c]{@{}c@{}} Hyper Parameter \end{tabular} & \begin{tabular}[c]{@{}c@{}}Optimal Value\end{tabular}  &  \begin{tabular}[c]{@{}c@{}}Range\end{tabular} 
\\ 

\midrule

\multirow{3}{*}{\begin{tabular}[c]{@{}l@{}}Rotated \& Fashion MNIST \\   Table \ref{tbl:mnist} (ResNet-18) \end{tabular}}    

& \begin{tabular}[c]{@{}c@{}} Total Epochs \end{tabular}                                                      &  - &  25   \\

& \begin{tabular}[c]{@{}c@{}} Learning Rate \end{tabular}                                                      &  0.01 &  [0.01]   \\

& \begin{tabular}[c]{@{}c@{}} Batch Size \end{tabular}                                                      &  16 &  [16]   \\

& \begin{tabular}[c]{@{}c@{}} Weight Decay \end{tabular}                                                      & 0.0005 &  [0.0005]   \\

& \begin{tabular}[c]{@{}c@{}} Match Penalty \end{tabular}                                                      &  0.1 &  [0.1, 1.0]   \\

& \begin{tabular}[c]{@{}c@{}} IRM Penalty \end{tabular}                                                      &  1.0 (RotMNIST); 0.05 (FashionMNIST)  &  [0.05, 0.1, 0.5, 1.0, 5.0]   \\

& \begin{tabular}[c]{@{}c@{}} IRM Threshold \end{tabular}                                                      &  5 (RotMNIST), 0 (FashionMNIST) &  [0, 5, 15, 20]   \\

\midrule

\multirow{3}{*}{\begin{tabular}[c]{@{}l@{}}Rotated MNIST \\   Table \ref{tbl:mnist-lenet} (LeNet) \end{tabular}}    

& \begin{tabular}[c]{@{}c@{}} Total Epochs \end{tabular}                                                      &  - &  100   \\

& \begin{tabular}[c]{@{}c@{}} Learning Rate \end{tabular}                                                      &  0.01 &  [0.01]   \\

& \begin{tabular}[c]{@{}c@{}} Batch Size \end{tabular}                                                      &  16 &  [16]   \\

& \begin{tabular}[c]{@{}c@{}} Weight Decay \end{tabular}                                                      & 0.0005 &  [0.0005]   \\

& \begin{tabular}[c]{@{}c@{}} Match Penalty \end{tabular}                                                      &  1.0 &  [0.1, 1.0]   \\

\midrule

\multirow{3}{*}{\begin{tabular}[c]{@{}l@{}}Rotated MNIST \\   Table \ref{tbl:mnist-domain-bed} (DomainBed) \end{tabular}}    

& \begin{tabular}[c]{@{}c@{}} Total Epochs \end{tabular}                                                      &  - &  25   \\

& \begin{tabular}[c]{@{}c@{}} Learning Rate \end{tabular}                                                      &  0.01 &  [0.01]   \\

& \begin{tabular}[c]{@{}c@{}} Batch Size \end{tabular}                                                      &  128 &  [16, 32, 64, 128]   \\

& \begin{tabular}[c]{@{}c@{}} Weight Decay \end{tabular}                                                      & 0.0005 &  [0.0005]   \\

& \begin{tabular}[c]{@{}c@{}} Match Penalty \end{tabular}                                                      &  1.0 &  [0.1, 1.0]   \\

\midrule

\multirow{3}{*}{\begin{tabular}[c]{@{}l@{}}PACS \\   Table \ref{tbl:pacs-full}, \ref{tbl:pacs-alexnet} \\ (ResNet-18, ResNet-50, AlexNet) \end{tabular}}    

& \begin{tabular}[c]{@{}c@{}} Total Epochs \end{tabular}                                                      &  - &  50   \\

& \begin{tabular}[c]{@{}c@{}} Learning Rate \end{tabular}                                                      &  Table \ref{tbl:hyper-param-pacs} &  [0.01, 0.001, 0.0005]   \\

& \begin{tabular}[c]{@{}c@{}} Batch Size \end{tabular}                                                      &  16 &  [16]   \\

& \begin{tabular}[c]{@{}c@{}} Weight Decay \end{tabular}                                                      & 0.0005 &  [0.0005]   \\

& \begin{tabular}[c]{@{}c@{}} Match Penalty \end{tabular}                                                      &  Table \ref{tbl:hyper-param-pacs} &  [0.01, 0.1, 0.5, 1.0, 5.0]   \\

\midrule

\multirow{3}{*}{\begin{tabular}[c]{@{}l@{}} Chest X-ray \\   Table \ref{tbl:chestxray} (DenseNet-121) \end{tabular}}    

& \begin{tabular}[c]{@{}c@{}} Total Epochs \end{tabular}                                                      &  - &  40   \\

& \begin{tabular}[c]{@{}c@{}} Learning Rate \end{tabular}                                                      &  0.001 &  [0.01, 0.001]   \\

& \begin{tabular}[c]{@{}c@{}} Batch Size \end{tabular}                                                      &  16 &  [16]   \\

& \begin{tabular}[c]{@{}c@{}} Weight Decay \end{tabular}                                                      & 0.0005 &  [0.0005]   \\

& \begin{tabular}[c]{@{}c@{}} Match Penalty \end{tabular}                                                      &  10.0 (\texttt{RandMatch}), 50.0 (\texttt{MatchDG}, \texttt{MDGHybrid}) &  [0.1, 1.0, 10.0, 50.0]   \\

& \begin{tabular}[c]{@{}c@{}} IRM Penalty \end{tabular}                                                      &  10.0  &  [0.1, 1.0, 10.0, 50.0]   \\

& \begin{tabular}[c]{@{}c@{}} IRM Threshold \end{tabular}                                                      &  5  &  [0, 5, 15, 20]   \\

\bottomrule 
\end{tabular}}
\label{tbl:hyperparam-details}
\end{table*}

\begin{table*}[t]
\footnotesize
\caption{\texttt{MatchDG} Phase 1 training details for all the datasets.We did not do hyper parameter tuning as we did for other methods, hence we mention the default value for each hyper parameter that we used. Please note we still did early stopping, the Total Epochs in the table reflects the max budget for training. The specific archiecture used for Phase 1 training is also mentioned for each dataset. }
\centering
\begin{tabular}{@{}l | l |c @{}}
\toprule
Dataset                                                                            & \begin{tabular}[c]{@{}c@{}} Hyper Parameter \end{tabular} & \begin{tabular}[c]{@{}c@{}}Default Value\end{tabular}  
\\ 

\midrule

\multirow{6}{*}{\begin{tabular}[c]{@{}l@{}}Rotated \& Fashion MNIST \\   Table \ref{tbl:mnist}, \ref{tbl:mnist-lenet},  \ref{tbl:mnist-domain-bed} \end{tabular}}    

& \begin{tabular}[c]{@{}c@{}} Total Epochs \end{tabular}                                                      &  50 \\

& \begin{tabular}[c]{@{}c@{}} Learning Rate \end{tabular}                                                      &  0.01   \\

& \begin{tabular}[c]{@{}c@{}} Batch Size \end{tabular}                                                      &  64 (Table \ref{tbl:mnist}), 512 (Table \ref{tbl:mnist-lenet},  Table \ref{tbl:mnist-domain-bed})  \\

& \begin{tabular}[c]{@{}c@{}} Weight Decay \end{tabular}                                                      & 0.0005   \\

& \begin{tabular}[c]{@{}c@{}} $\tau$ \end{tabular}                                                      &  0.05    \\

& \begin{tabular}[c]{@{}c@{}} Architecture \end{tabular}                                                      &  ResNet-18 (Table~\ref{tbl:mnist}), LeNet (Table~\ref{tbl:mnist-lenet}), Custom CNN (Table~\ref{tbl:mnist-domain-bed})   \\

\midrule

\multirow{6}{*}{\begin{tabular}[c]{@{}l@{}}PACS \\   Table \ref{tbl:pacs-full}, \ref{tbl:pacs-alexnet}  \end{tabular}}    

& \begin{tabular}[c]{@{}c@{}} Total Epochs \end{tabular}                                                      &  50 \\

& \begin{tabular}[c]{@{}c@{}} Learning Rate \end{tabular}                                                      &  0.01   \\

& \begin{tabular}[c]{@{}c@{}} Batch Size \end{tabular}                                                      &  32 \\

& \begin{tabular}[c]{@{}c@{}} Weight Decay \end{tabular}                                                      & 0.0005   \\

& \begin{tabular}[c]{@{}c@{}} $\tau$ \end{tabular}                                                      &  0.05    \\

& \begin{tabular}[c]{@{}c@{}} Architecture \end{tabular}                                                      &  ResNet-50   \\

\midrule

\multirow{6}{*}{\begin{tabular}[c]{@{}l@{}} Chest X-ray \\   Table \ref{tbl:chestxray} \end{tabular}}    

& \begin{tabular}[c]{@{}c@{}} Total Epochs \end{tabular}                                                      &  50 \\

& \begin{tabular}[c]{@{}c@{}} Learning Rate \end{tabular}                                                      &  0.01   \\

& \begin{tabular}[c]{@{}c@{}} Batch Size \end{tabular}                                                      &  32 \\

& \begin{tabular}[c]{@{}c@{}} Weight Decay \end{tabular}                                                      & 0.0005   \\

& \begin{tabular}[c]{@{}c@{}} $\tau$ \end{tabular}                                                      &  0.05    \\

& \begin{tabular}[c]{@{}c@{}} Architecture \end{tabular}                                                      &  DenseNet-121   \\

\bottomrule 
\end{tabular}
\label{tbl:mdg-p1}
\end{table*}

\subsubsection{MatchDG implementation details:}  The MatchDG algorithm proceeds in two phases.  \\
\textbf{Initialization}: 
We construct matches of pairs of same-class data points from different domains. Hence, given each data point we randomly select another data point with the same class from  another domain. The matching for each class across domains is done relative to a base domain; which is chosen by taking the domain that has the highest number of samples for that class. This is done to avoid missing out on data points when there is class imbalance across domains. Specifically, we iterate over classes and for each class, we match data points randomly across domains w.r.t a base domain for that class. This leads to matrix $\mathcal{M}$ of size $(N^{'}, K)$, where $N^{'}$ refers to the updated domain size ( sum of the size of base domain for all the classes ) and K refers to the total number of domains. We describe the two phases below: \\
    
\textbf{Phase 1}: We samples batches $(B,K)$ from the matched data matrix $\mathcal{M}$,  where B is the batch size. For each data point $x_{i}$ in the batch, we minimize the contrastive loss from ~\eqref{eq:ntloss} by selecting its matched data points across domains as the positive matches and consider every data point with a different class label from $x_{i}$ to be a negative match.
         
After every $t$ epochs, we periodically update the matched data matrix by using the representations learnt by contrastive loss minimization. We follow the same procedure of selecting a base domain for each class, but instead of randomly matching  data points across domains, we find the nearest neighbour for the data point in base domain among the data points in the other domains with the same class label based on the $l_2$ distance between their representations. 
  
At the end of Phase I,  we update the matched data matrix based on $l_2$ distance over the final representations learnt. We call these matches as the \emph{inferred} matches.  \\

\textbf{Phase 2}:
           We train using the loss from Eq. \eqref{eq:erm-pmatch-loss}, with the match function $\Omega$ based on the inferred matches generated from  Phase 1 (ERM + Inferred Match). We train the network from scratch in  Phase 2 and use the representations learnt in Phase 1 to only update the matched data matrix. 
           
           The updated data matrix based on representations learnt in Phase 1 may lead to many-to-one matches from the base domain to the other domains. This can lead to certain data points being excluded from the training batches. 
           Therefore, we construct batches such that each batch consists of two parts.
           The first is sampled as in Phase 1 from the matched data matrix. The second part is sampled randomly from  all train domains. 
Specifically, for each batch $(B, K)$ sampled from the matched data matrix, we sample an additional part of size $B$ with data points selected randomly across domains. The loss for the second part of the batch is simply ERM, 
along with ERM + InferredMatch Loss on the first part of the  batch.

\subsection{Metrics for evaluating quality of learnt matches} 
\label{sm:eval-metrics}
Here we describe the three metrics used for measuring overlap of the learnt matches with ground-truth ``perfect'' matches. 

\texttt{Overlap \%}: Percentage  of matches (j, k) as per the perfect match strategy $\Omega$ that are also consistent with the learnt match strategy $\Omega^{'}$.

\begin{equation}
\footnotesize
     \frac{\sum_{\Omega(j,k)=1; d\neq d'} \Omega^{'}(j,k)}{\sum_{\Omega(j,k)=1; d\neq d'} 1}
\end{equation}

\texttt{Top-10  Overlap \%}: Percentage of matches (j, k) as per the perfect match strategy $\Omega$ that are among the Top-10 matches for the data point j w.r.t the learnt match strategy $\Omega^{'}$ i.e. $S^{10}_{\Omega^{'}}(j)$
\begin{equation} 
\footnotesize
     \frac{\sum_{\Omega(j,k)=1; d\neq d'}  \mathbbm{1}[ k \in S^{10}_{\Omega^{'}}(j) ] }{\sum_{\Omega(j,k)=1; d\neq d'} 1}
\end{equation}    

\texttt{Mean Rank}: For the matches (j, k) as per the perfect match strategy $\Omega$, compute the mean rank for the data point j w.r.t the learnt match strategy $\Omega^{'}$ i.e. $S_{\Omega^{'}}(j)$

\begin{equation}
\footnotesize
     \frac{\sum_{\Omega(j,k)=1; d\neq d'}  Rank[ k \in S_{\Omega^{'}}(j) ] }{\sum_{\Omega(j,k)=1; d\neq d'} 1}
\end{equation}    

\subsection{HyperParameter Tuning}
\label{sm:hyperparam-tuning}
To select hyperparameters, prior works~\cite{dou2019masf, carlucci2019domain, li2018learningmetadg} use leave-one-domain-out validation, which means that the hyperparameters are tuned after looking at data from the unseen domain. Such a setup is violates the premise of the domain generalization task that assumes that a model should have no access to the test domain.  Therefore, in this work,  we construct a validation set using only the source domains and use it for hyper parameter tuning. In the case of PACS, we already have access to the validation indices for each domain and use them to construct a validation set based on the source domains. For Rotated \& Fashion MINST, Chest X-ray datasets, we create validation set for each source domain as described in the section B.1 above. Hence, the model does not have access to the data points from the target/test domains at the time of training and validation. 

We perform a grid search over pre-defined values for each hyper parameter and report the optimal values along with the values used for grid search in Table~\ref{tbl:hyperparam-details}. Further, we do early stopping based on the validation accuracy on source domains and use the models which obtain the best validation accuracy. 

For the case of \texttt{MatchDG} Phase-1, we do not perform grid search and use default values for each hyper parameter (Table \ref{tbl:mdg-p1}). We still do early stopping for \texttt{MatchDG} Phase-1, based on the metric \texttt{Top-10 Overlap} (Section B.2) over the validation set of source domains. Since we require perfect matches for the evaluation of the metric \texttt{Top-10 Overlap}, we create prefect matches using the self augmentations (Section B.1) for each dataset.

\begin{table*}[th!]
\footnotesize
\caption{Optimal values for hyper parameters on PACS. Batch Size (16), Weight Decay (0.0005) was consistent across different cases. The Match Penalty for the method \texttt{MDGHybrid} corresponds to (\texttt{MatchDG} penalty, \texttt{PerfMatch} penalty). }
\centering
\begin{tabular}{@{}l | l | l l l l l@{}}
\toprule
Architecture                                                                            & \begin{tabular}[c]{@{}c@{}} Hyper Parameter \end{tabular} & \begin{tabular}[c]{@{}c@{}}Test Domain\end{tabular} &
\begin{tabular}[c]{@{}c@{}}ERM\end{tabular} &
\begin{tabular}[c]{@{}c@{}}\texttt{RandMatch}\end{tabular}  &
\begin{tabular}[c]{@{}c@{}}\mdg (Phase 2) \end{tabular}  &
\begin{tabular}[c]{@{}c@{}}\texttt{MDGHybrid}\end{tabular}  
\\ 

\midrule

\multirow{8}{*}{\begin{tabular}[c]{@{}l@{}}ResNet-18 \\   Table \ref{tbl:pacs-r18}, \ref{tbl:pacs-full} \end{tabular}}    

& \multirow{4}{*}{\begin{tabular}[c]{@{}c@{}} Learning Rate \end{tabular}}
    & Photo & 0.001  & 0.001  & 0.0005 &  0.0005\\

& 
    & Art Painting & 0.01  & 0.01 & 0.001  & 0.001\\

& 
    & Cartoon & 0.01  & 0.001  & 0.001  & 0.001\\

& 
    & Sketch & 0.01  & 0.01  & 0.01   & 0.01\\\cmidrule{2-7}

& \multirow{4}{*}{\begin{tabular}[c]{@{}c@{}} Match Penalty \end{tabular}}       
& Photo & 0 & 5.0  & 1.0 & (0.1, 0.1)  \\

& 
& Art Painting &  0   & 0.1 & 5.0  &  (0.01, 0.1)  \\

& 
& Cartoon & 0   & 5.0 & 1.0   &  (0.1, 0.1)  \\

& 
& Sketch & 0   & 0.5 & 0.5  & (0.01, 0.1)  \\

\midrule

\multirow{8}{*}{\begin{tabular}[c]{@{}l@{}}ResNet-50 \\   Table \ref{tbl:pacs-r50}, \ref{tbl:pacs-full}  \end{tabular}}    

&\multirow{4}{*}{\begin{tabular}[c]{@{}c@{}} Learning Rate \end{tabular}}
    & Photo & 0.0005  & 0.0005  & 0.0005 &  0.0005\\

& 
    & Art Painting & 0.01  & 0.01 & 0.001  & 0.001\\

&
    & Cartoon & 0.01  & 0.01 & 0.001  & 0.0005\\

& 
    & Sketch & 0.01  & 0.01 & 0.0005   & 0.001\\ \cmidrule{2-7}

&\multirow{4}{*}{\begin{tabular}[c]{@{}c@{}} Match Penalty \end{tabular}}                                     & Photo & 0 & 5.0  & 0.01 & (0.1, 0.1)  \\

&
& Art Painting &  0   & 0.1   & 0.1  &  (0.01, 0.1)  \\

& 
& Cartoon & 0   & 0.01  & 0.01   &  (0.01, 0.1) \\

& 
& Sketch & 0   & 0.1  & 5.0  & (0.01, 0.1)    \\

\midrule

\multirow{8}{*}{\begin{tabular}[c]{@{}l@{}}AlexNet \\   Table \ref{tbl:pacs-alexnet}  \end{tabular}}    

& \multirow{4}{*}{\begin{tabular}[c]{@{}c@{}} Learning Rate \end{tabular}}
    & Photo & 0.0005  & 0.0005  & 0.0005 &  0.0005\\

& 
    & Art Painting & 0.001  & 0.001 & 0.001  & 0.001\\

& 
    & Cartoon & 0.001  & 0.001 & 0.001  & 0.001\\

& 
    & Sketch & 0.0005  & 0.001 & 0.001   & 0.001\\ \cmidrule{2-7}

& \multirow{4}{*}{\begin{tabular}[c]{@{}c@{}} Match Penalty \end{tabular}}                                    & Photo & 0 & 0.1  & 0.1 & (0.1, 0.1)  \\ 

& 
& Art Painting &  0   & 0.1   & 1.0  &  (0.01, 0.1)  \\

&
& Cartoon & 0   & 0.5  & 1.0   &  (0.01, 0.1) \\

& 
& Sketch & 0   & 0.5  & 0.1  & (0.01, 0.1)    \\

\bottomrule 
\end{tabular}
\label{tbl:hyper-param-pacs}
\end{table*}

\subsection{Reproducing Results from Prior Work}
\label{sm:prior-work}

\paragraph{MNIST and Fashion MNIST}
The results for MASF, CSD, and IRM in  Table~\ref{tbl:mnist} were computed using their code which is available online~\footnote{\label{csd-fnote}https://github.com/vihari/CSD }\footnote{https://github.com/biomedia-mira/masf}\footnote{\label{irm-fnote}https://github.com/facebookresearch/InvariantRiskMinimization}. The MASF code was hardcoded to run for  PACS dataset; which has 3 source domains that gets divided into 2 meta train and 1 meta test domain. Their code requires atleast 2 meta train domains; which leads to an issue for only 2 source domains (30, 45). In Table~\ref{tbl:mnist} when there are only 2 source domains; their code considers  only 1 meta train domain. To resolve this issue; we create a copy of the 1 meta train domain and thus run MASF for source domains 30, 45 on MNIST.

The results for prior approaches in  Table~\ref{tbl:mnist-lenet} are taken from \cite{shankar2018generalizing}, \cite{ilse2020diva}. For the results using DomainBed setup in Table~\ref{tbl:mnist-domain-bed}, the results for prior approaches are taken from \cite{gulrajani2020domainbed}.
\paragraph{PACS}
We did not generate results for the prior approaches for PACS by developing or using existing implementations.  All the results for the prior approaches on PACS were taken from the respective papers as specified in the Table~\ref{tbl:pacs-full}, \ref{tbl:pacs-alexnet}.

\paragraph{Chest X-ray}
The results for the prior approaches CSD, IRM were generated using the implementations of both of the methods available on github \textsuperscript{\ref{csd-fnote},\ref{irm-fnote}} .

\begin{table*}[ht]
    \caption{Accuracy for Rotated MNIST datasets using the LeNet architecture as proposed in \cite{motiian2017ccsa}. The results for the prior approaches CCSA \cite{motiian2017ccsa}, D-MTAE \cite{ghifary2015multitaskae}, LabelGrad \cite{goodfellow2014explaining}, DAN \cite{ganin2016dann}, and CrossGrad \cite{shankar2018generalizing} are taken from Table 9 in \cite{shankar2018generalizing}. The results for DIVA \cite{ilse2020diva} are taken from the Table 1 in their paper.}
\centering
\begin{tabular}{@{} l |c c c c c c c}
\toprule
 
\begin{tabular}[c]{@{}l@{}}Algorithm \end{tabular} & 0  & 15  & 30  & 45  & 60 & 75 & Average \\ 
 
 \midrule
 ERM & 88.2 (1.0) & 98.6 (0.5) & 97.7 (0.6) & 97.5 (0.3) & 97.0 (0.1) & 85.6 (2.1) & 94.1 \\
 CCSA & 84.6  & 95.6 & 94.6 & 82.9 & 94.8 & 82.1 & 89.1 \\                          D-MTAE &  82.5 & 96.3  & 93.4  & 78.6  & 94.2  & 80.5 & 87.6  \\  
 LabelGrad & 89.7  & 97.8  & 98.0  & 97.1 & 96.6 & 92.1 & 95.2\\
 DAN & 86.7  & 98.0 & 97.8 & 97.4 & 96.9 & 89.1 & 94.3\\                  
 CrossGrad & 88.3 & 98.6 & 98.0  & 97.7 & 97.7 & 91.4 &  95.3\\
 DIVA & \textbf{93.5} (0.3) & 99.3 (0.1) & 99.1 (0.1) & 99.2 (0.1) & 99.3 (0.1) & 93.0 (0.4) & 97.2\\               
 \texttt{RandMatch} & 91.0 (0.9)  & \textbf{99.7} (0.2) & 99.6 (0.1) & \textbf{99.4} (0.1) & \textbf{99.7} (0.1) & 93.1 (1.1) & 97.1 \\
 \texttt{MatchDG} & 93.0 (0.5) & 99.5 (0.3)  & \textbf{99.9} (0.1) & \textbf{99.4} (0.1) & \textbf{99.7} (0.3) & \textbf{93.3} (1.1) & \textbf{97.4}\\
 
 \midrule
 
 \texttt{PerfMatch} & 96.5 (0.6) & 99.1 (0.3)  & 99.2 (0.3) & 98.6 (0.7) & 98.6 (1.0) & 94.9 (1.8) &  97.8 \\
 
\bottomrule 
\end{tabular}
\label{tbl:mnist-lenet}
\vspace{-0.3cm}
\end{table*}

\begin{table*}[h]
    \caption{Accuracy for Rotated MNIST datasets using the DomainBed setup as proposed in \cite{gulrajani2020domainbed}. The results for the approaches IRM~\cite{arjovsky2019irm}, DRO~\cite{sagawa2019distributionally}, Mixup~\cite{xu2019adversarial, yan2020improve, wang2020heterogeneous}, MLDG~\cite{li2018learningmetadg}, CORAL~\cite{sun2016deep}, MMD~\cite{li2018adversarialfeature}, DANN~\cite{ganin2016dann}, C-DANN~\cite{li2018adversarialcada} are taken from \cite{gulrajani2020domainbed}. }
\centering
\begin{tabular}{@{} l |c c c c c c c}
\toprule
 
 \begin{tabular}[c]{@{}l@{}}Algorithm \end{tabular} & 0  & 15  & 30  & 45  & 60 & 75 & Average  \\ 
 
 \midrule
 
 ERM & 95.6 (0.1) & 99.0 (0.1) & 98.9 (0.0) & 99.1 (0.1) & \textbf{99.0} (0.0)  &  \textbf{96.7} (0.2) & 98.0 \\ 
 IRM & 95.9 (0.2) & 98.9 (0.0) & 99.0 (0.0) & 98.8 (0.1) & 98.9 (0.1) & 95.5 (0.3) & 97.9\\                                                             
 DRO & 95.9 (0.1) & 98.9 (0.0) & 99.0 (0.1) & 99.0 (0.0) &\textbf{99.0} (0.0) & 96.9 (0.1) & \textbf{98.1} \\                                                            
 Mixup & 96.1 (0.2) & \textbf{99.1} (0.0) & 98.9 (0.0) & 99.0 (0.0) & 99.0 (0.1) & 96.6 (0.1)  & \textbf{98.1}\\
 MLDG & 95.9 (0.2) & 98.9 (0.1) & 99.0 (0.0) & \textbf{99.1} (0.0) & \textbf{99.0} (0.0) & 96.0 (0.2) & 98.0\\                                                            
 CORAL & 95.7 (0.2) & 99.0 (0.0) & \textbf{99.1} (0.1) & \textbf{99.1} (0.0) & \textbf{99.0} (0.0) & \textbf{96.7} (0.2) & \textbf{98.1}\\
 
 MMD & \textbf{96.6} (0.1) & 98.9 (0.0) & 98.9 (0.1) & 99.1 (0.1) & \textbf{99.0} (0.0) & 96.2 (0.1) & \textbf{98.1} \\                        
 DANN & 95.6 (0.3) & 98.9 (0.0) & 98.9 (0.0) & 99.0 (0.1) & 98.9 (0.0) & 95.9 (0.5) & 97.9 \\                                               C-DANN & 96.0 (0.5) & 98.8 (0.0) & 99.0 (0.1) & \textbf{99.1} (0.0) & 98.9 (0.1) & 96.5 (0.3) & 98.0 \\     
 
 \texttt{RandMatch} & 95.4 (0.4) & 98.2 (0.1) & 97.9 (0.5) & 98.5 (0.1) & 98.1 (0.1) & 94.3 (0.3) & 97.1 \\
 \texttt{MatchDG} & 95.9 (0.1) & 98.4 (0.1) & 98.6 (0.2)  & 98.9 (0.2) & 98.7 (0.1) & 95.1 (0.3) & 97.6 \\
 
\bottomrule 
\end{tabular}
\label{tbl:mnist-domain-bed}
\vspace{-0.3cm}
\end{table*}

\section{Additional Evaluation on Rotated MNIST and Fashion-MNIST}
Here we present results for additional experiments on Rotated MNIST and Fashion-MNIST datasets using \mdg.

\subsection{Comparing \mdg\ with prior work on the LeNet Network}
\label{sm:2-layer}

Table~\ref{tbl:mnist-lenet} compares the accuracy results for \mdg\ with prior work on the LeNet architecture~\cite{motiian2017ccsa}. In this setup, there are six domains in total ($0^\circ, 15^\circ, 30^\circ, 45^\circ, 60^\circ, 75^\circ$). For each test domain, the remaining five domains are used as source training domains.   We observe that matching-based training methods \texttt{RandMatch} and \mdg\ outperform prior work on the all the domains except the test domain 0, where \mdg\ is competitive to the best performing approach DIVA. They also achieve accuracy almost equal to the oracle case \texttt{PerfMatch}  for target angles ($15^\circ$ to $60^\circ$) that lie in between the source domains.

\subsection{Comparing \mdg\ on Domain Bed Benchmark}
\label{sm:mnist-domain-bed}
Table \ref{tbl:mnist-domain-bed} compares the accuracy results for \mdg\ with prior work on the setup proposed by \cite{gulrajani2020domainbed}. This setup is similar to the setup in the section~\ref{sm:2-layer}, however, it uses a custom CNN architecture and all the $70,000$ images for each domain. For a fair comparison, we use the same custom CNN architecture for learning the match function during the \mdg\ Phase-I. Even under this constraint, \mdg\ average accuracy is only 0.5\% percent behind the best performing approaches (CORAL, MMD). As supported by our experiments before (Table~\ref{tbl:mnist}, ~\ref{tbl:metrics}) we believe that using more powerful architectures (ResNet-18, ResNet-50) during the \mdg\ Phase-1 should help in learning a better match function and consequently better average accuracy. 

\subsection{Accuracy Results using a fraction of perfect matches}
\label{sm:frac-matches}

To show the importance of learning a good match function, we present the results of approaches with match function capturing some fixed percentage of perfect matches in the Table~\ref{tbl:fraction}. For both Rotated \& Fashion MNIST, we observe that the approaches that contain a higher proportion of perfect matches perform better in terms of accuracy on target domains. Hence, the quality of the match function leads to monotonic effect on the generalization performance of the matching approaches. 

\begin{table}[th]
\caption{Accuracy results using a fraction of perfect matches during training}
\centering
\begin{tabular}{@{}l | c c@{}}
\toprule
                                                                & MNIST       & \begin{tabular}[c]{@{}l@{}}Fashion-MNIST\end{tabular} \\ \midrule
\begin{tabular}[c]{@{}l@{}}RandMatch\end{tabular}          &  93.4 (0.26) & 77.0 (0.42)                                        \\ \midrule
Approx 25\%                                                     & 93.8 (0.48) & 77.8 (0.79)                                             \\
Approx 50\%                                                     & 94.0 (0.42) & 78.0 (0.78)                                             \\
Approx 75\%                                                     & 94.7 (0.14) & 78.9 (0.31)                                             \\ \midrule
\begin{tabular}[c]{@{}l@{}}PerfMatch (100\%)\end{tabular} & 96.0 (0.41) & 81.6 (0.46)                                             \\ \bottomrule
\end{tabular}
\label{tbl:fraction}
\end{table}

\subsection{Quality of representation learnt in the  classification phase}
\label{sm:metrics-phase2}
In addition to Table~\ref{tbl:metrics} that shows metrics for Phase 1 of \mdg, we compute the metrics for the classification phase (Phase 2) of \mdg.  Specifically, we compute the Overlap, Top-10 overlap and the Mean Rank metrics (\secref{sm:eval-metrics}) for matched pairs of inputs based on the representation learnt at the end of the  classification phase. 

Table~\ref{tbl:sm-metrics} shows the matching metrics for \mdg\ and compares it to the matches based on the representations (last layers) learnt by the  \epmatch\ and \ermatch\ methods. For both Rotated-MNIST and Fashion-MNIST datasets, \mdg\ obtains mean rank, Top 10 overlap and total overlap between \epmatch\ and \ermatch. As the Fashion-MNIST dataset is more complex than the digits dataset, we observe that the mean rank with different training techniques is higher than the corresponding values  for the Rotated-MNIST dataset.

\subsection{Matching metrics for Fashion-MNIST dataset with 2000 training samples per domain}
\label{sm:fmnist-lower}
In the main text (\tabref{tbl:metrics}), we computed matching metrics for MatchDG (Phase 1) over the Fashion-MNIST dataset with $10000$ samples per domain. Here  we compute the same metrics for a smaller dataset with 2000 samples per domain. 

We compute the metric for the default instantiation of Phase 1 of \mdg\ initialized with random matches and compare it to an {\em oracle} version of   \mdg\ initialized with perfect matches. In addition, we compare the metrics for matches generated using baseline 
ERM (last layer of the network) in order to understand its effectiveness as a matching strategy in Phase 1. Table~\ref{tbl:metrics-abl-2} shows the metrics for Phase 1 of \mdg\ with 2K images from the Fashion-MNIST dataset, and reproduces the metrics for the 10K dataset from \tabref{tbl:metrics} for ease of comparison. 
We observe that the mean rank of perfect matches improves for the smaller dataset. Similarly, the overlap and top-10 overlap also increase for the smaller dataset. A possible reason is that there are fewer alternative matches to the perfect match as the number of samples is reduced. 
That said, while the overlap with perfect matches may decrease as sample size increases,  the accuracy of the resultant classifier may still increase due to higher sample size. 

\begin{table}[tb]
\caption{Mean rank, Top-10 overlap, and overlap metrics for the matches learnt in the classification phase (Phase 2), when trained on all five source domains in the Rotated MNIST and FashionMNIST datasets.}
 \centering
\footnotesize
 
 \resizebox{\linewidth}{!}{
\begin{tabular}{@{}l | l | c c c @{}}
\toprule
Dataset                                                                            & \begin{tabular}[c]{@{}c@{}}Method \end{tabular} & Overlap (\%) &  \begin{tabular}[c]{@{}c@{}}Top 10 Overlap (\%) \end{tabular} & \begin{tabular}[c]{@{}c@{}}  Mean Rank \end{tabular} 
\\ \midrule
\multirow{3}{*}{\begin{tabular}[c]{@{}l@{}}Rotated\\ MNIST \end{tabular}}        
                                                                                   & RandMatch                                                    & 2.2 (0.18) & 13.5 (0.36) &  75.5 (1.65)   \\
                                                                                & {\begin{tabular}[l]{@{}l@{}}MatchDG\\ (Phase 2)        \end{tabular}}                                              & { \bf 17.7 }(0.97)  & { \bf 41.8 (2.89) } &  { \bf 39.6 (3.58) }  \\ \cmidrule {2-5}
                                                                                   & \begin{tabular}[l]{@{}l@{}}PerfMatch \\ (Oracle) \end{tabular}                                                  & 78.2 (1.91)  & 95.5 (1.37)  & 1.84 (0.67)    \\ 
                                                                                    \midrule
\multirow{3}{*}{\begin{tabular}[c]{@{}l@{}}Fashion\\ MNIST (10k) \end{tabular}}        
                                                                                   & RandMatch                                                    & 0.5 (0.04)  & 3.2 (0.17)  & 420.0 (7.27)    \\
                                                                                              & {\begin{tabular}[l]{@{}l@{}}MatchDG\\ (Phase 2)        \end{tabular}}                                             & { \bf 1.8 (0.13) }  & { \bf 8.5 (0.56) } & { \bf 296.5 (9.94)} \\ \cmidrule {2-5}
                                                                                   & \begin{tabular}[l]{@{}l@{}}PerfMatch \\ (Oracle) \end{tabular}                                                       &  9.2 (0.21) & 30.5 (0.38)  & 114.7 (3.29)    \\ 
                              \bottomrule 
\end{tabular}
}
\label{tbl:sm-metrics}
\end{table}

\subsection{Iterative updating of matches in Phase-1 of MatchDG}
\label{sm:iterative-ctr}
In Section~\ref{sec:two-phase}, we proposed Phase 1 of the \mdg\ algorithm with iterative updates to the computed matches. Here we compare the quality of matches learnt at the end of Phase 1 with or without using the iterative updating. Without the iterative updates, the matches always remain the same as the random matches with which the algorithm was initialized.   

Table~\ref{tbl:metrics-abl} shows metrics computed at the end of Phase 1 of \mdg\ using both an  iterative approach vs. a non-iterative approach. The iterative approach provides a $2\times$ improvement on the overlap with perfect matches for rotated MNIST and Fashion-MNIST datasets. Since higher overlap in the inferred matches results in better classification accuracy in Phase 2 (as shown in \tabref{tbl:fraction}), we conclude that using the iterative approach improves the domain generalization capability of \mdg.

\begin{table}[h]
\footnotesize
    \caption{Metrics computed at \mdg\ (Phase 1) for Fashion-MNIST dataset with 2K and 10K sample size used for training. Lower is better for mean rank.
}
\centering
\resizebox{\linewidth}{!}{
\begin{tabular}{@{}l | l |c c c @{}}
\toprule
Dataset                                                                            & \begin{tabular}[c]{@{}c@{}}Method \end{tabular} & \begin{tabular}[c]{@{}c@{}}Overlap (\%)\end{tabular}  &  \begin{tabular}[c]{@{}c@{}}Top 10 Overlap (\%) \end{tabular} &  Mean Rank
\\ 

\midrule

\multirow{3}{*}{\begin{tabular}[c]{@{}l@{}}Fashion\\ MNIST (2k) \end{tabular}}    

& \begin{tabular}[c]{@{}c@{}} ERM \end{tabular}                                                      &  8.7 (0.14) &  36.0 (1.41) & 36.1 (1.66)   \\

& \begin{tabular}[c]{@{}c@{}}\mdg\ \\  (Default) \end{tabular}                                                      &  { \bf 38.5} (2.11) &  { \bf 71.2} (0.91) & { \bf 15.9} (0.54)   \\\cmidrule {2-5}

 & \begin{tabular}[l]{@{}l@{}}\mdg\  \\ (PerfMatch) \end{tabular}                                                         & 69.5 (10.8)  & 91.9 (5.8)  &  3.3 (2.4)    \\ 

\midrule 

\multirow{3}{*}{\begin{tabular}[c]{@{}l@{}}Fashion\\ MNIST (10k) \end{tabular}}           & \begin{tabular}[c]{@{}l@{}}
                                         ERM\end{tabular}  & 2.1 (0.12)  & 11.1 (0.63)   & 224.3 (8.73)  \\ 
                                                                                   
                                                                                              & \begin{tabular}[c]{@{}c@{}}\mdg \\ (Default) \end{tabular}                                                      & {\bf 17.9} (0.62) &  {\bf 43.1} (0.83) & {\bf 89.0} (3.15)   \\\cmidrule {2-5}
                                                                                   & \begin{tabular}[l]{@{}l@{}}\mdg \\ (PerfMatch) \end{tabular}                                                         & 56.2 (1.79)  & 87.2 (1.48)  &  7.3 (1.18)    \\ 
                              \bottomrule 
\end{tabular}}
\label{tbl:metrics-abl-2}
\end{table}

\begin{table}[tb]
\footnotesize
\caption{Overlap with perfect matches, top-10 overlap and the mean rank for perfect matches for Iterative and Non Iterative \mdg\ over all training domains. Lower is better for mean rank.
}
\centering
 \resizebox{\linewidth}{!}{
\begin{tabular}{@{}l | l |c c c @{}}
\toprule
Dataset                                                                            & \begin{tabular}[c]{@{}c@{}}Method (Phase 1) \end{tabular} & \begin{tabular}[c]{@{}c@{}}Overlap (\%)\end{tabular}  &  \begin{tabular}[c]{@{}c@{}}Top 10 Overlap (\%) \end{tabular} &  Mean Rank
\\ 

\midrule

\multirow{3}{*}{\begin{tabular}[c]{@{}l@{}} MNIST\end{tabular}}
                              
& \begin{tabular}[l]{@{}l@{}} \mdg\ \\  (Iterative) \end{tabular}  

& {\bf 28.9} (1.24) & {\bf 64.2} (2.42)  & {\bf 18.6} (1.59) \\ 
                                                                                  
& \begin{tabular}[l]{@{}l@{}} \mdg\ \\ (Non Iterative) \end{tabular}        
& 12.3 (0.28) & 37.9    (0.27)  & 37.8 (0.47)  \\ 

\midrule \midrule

\multirow{3}{*}{\begin{tabular}[c]{@{}l@{}}Fashion\\ MNIST (10k)\end{tabular}}

& \begin{tabular}[l]{@{}l@{}} \mdg\ \\ (Iterative) \end{tabular} 
            & {\bf 17.9} (0.62) &  {\bf 43.1} (0.83) & {\bf 89.0} (3.15)  
             \\\cmidrule {2-5}

& \begin{tabular}[l]{@{}l@{}} \mdg\ \\ (Non Iterative) \end{tabular}                                                      &  7.7 (0.28) &  23.8 (0.78) & 153.9 (9.63)   \\
 
\bottomrule 

\end{tabular}}
\label{tbl:metrics-abl}
\end{table}

\begin{figure*}[tb]
    \centering
    \includegraphics[scale=0.38]{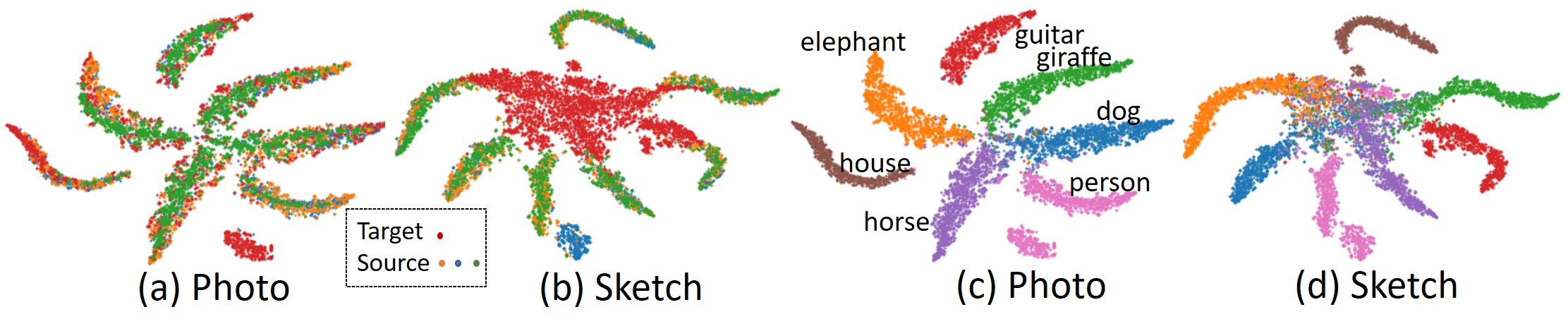}
    \caption{The t-SNE plots for visualizing features learnt in \mdg\ Phase 1. (a)-(c) are for Photo as the target domain  and (b)-(d) are for Sketch.
    }
    \label{fig:t-sne}
\end{figure*}

\section{Additional Evaluation on PACS}
\subsection{ResNet Results}
\label{sm:resnet}

Table~\ref{tbl:pacs-full} extends the evaluation on PACS with ResNet-18, ResNet-50 (Table~\ref{tbl:pacs-r18}, \ref{tbl:pacs-r50}) in the main text by adding comparison with more prior approaches. We observe that \texttt{MDGHybrid} beats most of the prior approaches on both the ResNet-18 and ResNet-50 evaluation, except DDEC~\cite{asadi2019shapedg}, and RSC~\cite{huang2020self}. However, as stated in the main paper, DDEC~\cite{asadi2019shapedg} rely on data from additional source like Behance BAM! dataset and we are also not sure about the validation mechanism used by them. If the validation mechanism used by them includes data from the target domain during validation, then \texttt{MatchDG} (Test), \texttt{MDGHybrid} (Test) obtain better accuracy than them.

\subsection{AlexNet Results}
\label{sm:alexnet}

Finally, we compare \texttt{RandMatch} and \mdg\ to prior work on generalization accuracy  for the PACS dataset using the AlexNet architecture. As in Table~\ref{tbl:pacs-full}, the task is to generalize to a test domain after training on the remaining three domains.

For all test domains, Table~\ref{tbl:pacs-alexnet} shows that both \texttt{RandMatch} and \mdg\ outperform the baseline ERM method. Averaging over the test domains, \texttt{MDGHybrid} provides improvement over \mdg\ (70.46 versus 69.91).  Moreover, on average \mdg, \texttt{MDGHybrid} are better than many previous approaches D-MTAE~\cite{ghifary2015multitaskae}, DBADG~\cite{li2017pacs}, CIDDG~\cite{li2018adversarialcada}, HEX~\cite{wang2019learning} and FeatureCritic~\cite{li2019feature}, but some other methods like MASF~\cite{dou2019masf}, DGER~\cite{zhao2020domain}, RSC~\cite{huang2020self} achieve higher accuracy than \mdg. Since  \mdg\ outperforms most of the prior work on the same dataset  when trained using ResNet-18, ResNet-50 architecture (\tabref{tbl:pacs-full}),  we speculate that \mdg\ requires a powerful underlying network architecture to use matches effectively for classification.

\subsection{T-SNE Plots}
\label{sm:t-sne}
Beyond accuracy, we investigate the quality of representations learnt by MatchDG using t-SNE~\cite{maaten2008visualizing} in Figure~\ref{fig:t-sne}. 
Comparing the Phase I models for the easiest (\emph{Photo}) and hardest (\emph{Sketch}) unseen domains (Figs.~\ref{fig:t-sne}a,b), we find that \mdg\ achieves a higher overlap  between train and test domains  for \emph{Photo} than \emph{Sketch}, highlighting the difficulty of generalizing to the \emph{Sketch} domain, even as classes are well-separated in the training domains for both models (Figs.~\ref{fig:t-sne}c,d).

\begin{table*}[ht]
\centering
\caption{Accuracy on PACS with ResNet 18 (default, top row set), Resnet 18 with test domain validation (middle row set), and  ResNet 50 (bottom row set). The results for JiGen~\cite{carlucci2019domain}, S-MLDG~\cite{li2020sequential}, D-SAM~\cite{d2018domain}, MMLD~\cite{matsuura2020domain}, DDAIG~\cite{zhou2020deep} SagNet~\cite{nam2019reducing},
DDEC~\cite{asadi2019shapedg},
DANN~\cite{ganin2016dann}, C-DANN~\cite{li2018adversarialcada}, DRO~\cite{sagawa2019distributionally},
Mixup~\cite{xu2019adversarial, yan2020improve, wang2020heterogeneous},
IRM~\cite{arjovsky2019irm}, MLDG~\cite{li2018learningmetadg}, MMD~\cite{li2018adversarialfeature}, CORAL~\cite{sun2016deep},  were taken from the DomainBed~\cite{gulrajani2020domainbed} paper. For G2DM~\cite{albuquerque2020generalizing}, DGER~\cite{zhao2020domain}, CSD~\cite{piratla2020efficient}, MASF~\cite{dou2019masf},  EpiFCR~\cite{li2019episodic}, MetaReg~\cite{balaji2018metareg}, RSC~\cite{huang2020self} it was taken from their respective paper.
}
\label{tbl:pacs-full}
\begin{tabular}{@{}l | cccc| c@{}}
\toprule
          & P            & A         & C        & S            & Average.  \\ \midrule
ERM       & 95.38 (0.86) & 77.68 (0.35) & 78.98 (0.59) & 74.75 (1.70) & 81.70 \\
JiGen     & 96.0        & 79.42        & 75.25        & 71.35        & 80.41 \\
MASF      & 94.99 (0.09)       & 80.29 (0.18)   & 77.17 (0.08)       & 71.69 (0.22)        & 81.04 \\
G2DM       & 93.75        & 77.78        & 75.54        & 77.58        & 81.16 \\
DGER       & 96.65 (0.21)        & 80.70 (0.71)        & 76.40 (0.34)        & 71.77 (1.27)        & 81.38 \\
CSD       & 94.1 (0.2)       & 78.9 (1.1)       & 75.8 (1.0)       & 76.7 (1.2)       & 81.4 
\\
EpiFCR      & 93.9        & 82.1    & 77.0        & 73.0        & 81.5 \\
MetaReg   & 95.5 (0.24) & 83.7 (0.19) & 77.2 (0.31) & 70.3 (0.28) & 81.7 \\
S-MLDG      & 94.80        & 80.50    & 77.80       & 72.80        & 81.50 \\
D-SAM      & 94.30        & 79.48    & 77.13       & 75.30        & 81.55 \\
MMLD      & 96.09        & 81.28    & 77.16        & 72.29        & 81.83 \\
DDAIG     & 95.30     &  \textbf{84.20}        & 78.10         & 74.70         &  83.10 \\  
SagNet     & 95.47      &  83.58        & 77.66         & 76.30         &  83.25 \\  
DDEC       & \textbf{96.93}        & 83.01        & 79.39        & 78.62       & 84.46 \\ 
RSC       & 95.99        & 83.43        & 80.31        & \textbf{80.85}        & \textbf{85.15} \\ 
\texttt{RandMatch} & 95.37 (0.25) & 78.16 (1.51) &  78.83 (1.18) & 75.13 (1.90) &   81.87 \\
\texttt{MatchDG}   & 95.93 (0.21)  & 79.77 (0.12)  & 80.03 (0.03) & 77.11 (0.35) &  83.21 \\ 
\texttt{MDGHybrid}   & 96.15 (0.40)  & 81.71 (0.75)  & \textbf{80.75} (0.50) & 78.79 (1.25) & 84.35  \\ 

\midrule

G2DM    (Test)   & 94.63        & 81.44        & 79.35        & 79.52        & 83.34 \\
\texttt{RandMatch} (Test) & 95.57 (0.17) & 79.09 (1.09) &  79.37 (0.89) & 77.60 (0.87) &  82.91  \\  
\texttt{MatchDG} (Test)   & 96.53 (0.05)  &  81.32 (0.38) & 80.70 (0.54) & 79.72 (1.01)  & 84.56   \\ 
\texttt{MDGHybrid} (Test)   & \textbf{96.67} (0.20)  & \textbf{82.80} (0.32)  &  \textbf{81.61} (0.06) & \textbf{81.05} (1.01)  & \textbf{85.53} \\ 

\midrule

DomainBed    (ResNet50)   & 97.8 (0.0)        & 88.1 (0.1)        & 77.9 (1.3)        & 79.1 (0.9)        &  85.7 \\
MASF (ResNet50)      & 95.01 (0.10)       & 82.89 (0.16)   & 80.49 (0.21)       & 72.29 (0.15)        & 82.67 \\
C-DANN (ResNet50)   & 97.0 (0.4)       &  84.0 (0.9)    & 78.5 (1.5)      & 71.8 (3.9) &  82.8 \\
MetaReg (ResNet50)   & 97.6 (0.31)       &  87.2 (0.13)    & 79.2 (0.27)      & 70.3 (0.18) &  83.6 \\
DRO (ResNet50)   & 98.0 (0.3)       &  86.4 (0.3)    & 79.9 (0.8)      & 72.1 (0.7) &  84.1 \\
Mixup (ResNet50)   & 97.7 (0.2)       &  86.5 (0.4)    & 76.6 (1.5)      & 76.5 (1.2) &  84.3 \\
IRM    (ResNet50)   & 96.7 (0.3)       &  85.0 (1.6)    & 77.6 (0.9)      & 78.5 (2.6) &  84.4 \\
DANN (ResNet50)   & 97.6     (0.2)       &  85.9 (0.5)    & 79.9 (1.4)      & 75.2 (2.8) &  84.6 \\
MLDG (ResNet50)   & 97.0 (0.9)       &  \textbf{89.1} (0.9)    & 78.8 (0.7)      & 74.4 (2.0) &  84.8 \\
MMD (ResNet50)   & 97.5 (0.4)       &  84.5 (0.6)    & 79.7 (0.7)      & 78.1 (1.3) &  85.0 \\
DGER       & 98.25 (0.12)        & 87.51 (1.03)        & 79.31 (1.40)        & 76.30 (0.65)        & 85.34 \\
CORAL    (ResNet50)   & 97.6 (0.0)  &  87.7 (0.6)    & 79.2 (1.1)      & 79.4 (0.7)   &  86.0 \\
RSC    (ResNet50)   & 97.92  &  87.89  & 82.16  & \textbf{83.35}   &  \textbf{87.83} \\
\texttt{RandMatch} (ResNet50) & 97.89 (0.11    ) & 82.16 (0.19) &  81.68 (0.45) & 80.45 (0.19) & 85.54  \\
\texttt{MatchDG} (ResNet50)   & 97.94 (0.27)  & 85.61 (0.81)  & 82.12 (0.69)  &  78.76 (1.13) & 86.11  \\  

\texttt{MDGHybrid} (ResNet50)   & \textbf{98.36} (0.06) & 86.74 (1.01)  & \textbf{82.32 } (0.76) & 82.66 (0.48)  & 87.52 \\  

\bottomrule
\end{tabular}
\end{table*}

\begin{table*}[h]

\centering
\caption{Accuracy results on the PACS dataset trained with Alexnet (default, top row set), and Alexnet with test domain validation (bottom row set). The results for  DBADG \cite{li2017pacs}, MTSSL~\cite{albuquerque2020improving}, CIDDG \cite{li2018adversarialcada}, HEX~\cite{wang2019learning}, FeatureCritic~\cite{li2019feature}, MLDG \cite{li2018learningmetadg}, REx~\cite{krueger2020out}, CAADG~\cite{rahman2020correlation}, Epi-FCR \cite{li2019episodic}, MASF \cite{dou2019masf} were taken from the DomainBed~\cite{gulrajani2020domainbed} paper. The results for DANN~\cite{ganin2016dann}, IRM \cite{arjovsky2019irm},  G2DM \cite{albuquerque2020generalizing} were taken from the G2DM paper. The results for D-MTAE \cite{ghifary2015multitaskae}, MetaReg \cite{balaji2018metareg}, JiGen \cite{carlucci2019domain}, DGER~\cite{zhao2020domain} and RSC~\cite{huang2020self} were taken from the respective paper.}
\begin{tabular}{@{}l | cccc| c@{}}
\toprule
           Method/Test Domain             & Photo       & Art Painting & Cartoon     & Sketch      & Average \\ \midrule
ERM                     & 85.29 (0.22) & 64.23 (0.18)  & 66.61 (0.88) & 59.25 (0.83) & 68.85    \\
D-MTAE                  & \textbf{91.12}       & 60.27        & 58.65       & 47.68       & 64.45   \\
DANN    & 88.10     & 63.20        & 67.50       & 57.00       & 69.00   \\
DBADG                   & 89.50       & 62.86        & 66.97       & 57.51       & 69.21   \\
MTSSL & 84.31       & 61.67           & 67.41        & 63.91 
    &   69.32  \\
CIDDG                   & 78.65       & 62.70        & 69.73       & 64.45       & 69.40   \\
HEX                     & 87.90       & 66.80        & 69.70       & 56.30       & 70.20   \\
FeatureCritic                     & 90.10       & 64.40        & 68.60       & 58.40       & 70.40   \\
MLDG                    & 88.00       & 66.23        & 66.88       & 58.96       & 70.71   \\
REx                    & 89.74       & 67.04        & 67.97       & 59.81       & 71.14   \\
CAADG                    & 89.16       & 65.52        & 69.90       & 63.37       & 71.98   \\
Epi-FCR                 & 86.1        & 64.7         & 72.3        & 65.0        & 72.0      \\
IRM                     & 89.97       & 64.84        & 71.16       & 63.63       & 72.39   \\
MetaReg                 & 91.07 (0.41)      & 69.82 (0.76)        & 70.35 (0.63)      & 59.26 (0.31)       & 72.62   \\
JiGen                   & 89.00       & 67.63        & 71.71       & 65.18       & 73.38   \\
G2DM                   & 88.12       & 66.60        & 73.36       & 66.19       & 73.55   \\
MASF                    & 90.68      & 70.35        & 72.46       & 67.33        & 75.21  \\
DGER                    & 89.92 (0.42)      & 71.34 (0.87)        & 70.29 (0.77)       & { \bf 71.15 (1.01) }       & 75.67  \\
RSC                    & 90.88      & { \bf 71.62 }        & \textbf{75.11}       & 66.62       & { \bf 76.05 }  \\
\texttt{RandMatch} & 85.42 (0.52) & 65.54 (1.14)  & 68.41 (1.62) & 59.46 (1.35) & 69.71    \\
\mdg     & 85.41 (0.41) & 66.21 (0.64)  & 68.47 (1.10) & 59.56 (1.24) &  69.91  \\ \texttt{MDGHybrid}     & 85.67 (0.67) & 66.89 (1.23)  & 68.89 (1.08) & 60.39 (2.13) &  70.46 \\ \midrule

\texttt{RandMatch} (Test) & 86.04 (0.47) & 67.35 (0.32)  & 69.71 (0.56) & 64.66 (1.08) & 71.94  \\
\mdg (Test) & 86.52 (0.43) & 67.99 (0.56)  & 69.92 (0.09) & 65.64 (1.48) & 72.52  \\
\texttt{MDGHybrid} (Test) & 87.03 (0.29) & 67.97 (0.79)  & 71.06 (0.43) & 67.19 (0.44) & 73.31   \\

\bottomrule
\end{tabular}
\label{tbl:pacs-alexnet}

\end{table*}





\end{document}
